\documentclass[parskip=full]{scrartcl}
\usepackage[utf8]{inputenc}
\usepackage{authblk}
\usepackage{xcolor}
\usepackage{hyperref}

\title{Global Capacity Measures for Deep ReLU Networks via Path Sampling}
\author[1]{Ryan Theisen\thanks{Work conducted while at Salesforce Research. Email: \href{mailto:theisen@berkeley.edu}{theisen@berkeley.edu}}}
\author[2]{Jason M. Klusowski\thanks{Supported in part by NSF Grant DMS-1915932. Email: \href{mailto:jason.klusowski@rutgers.edu}{jason.klusowski@rutgers.edu}}}
\author[3]{Huan Wang}
\author[3]{Nitish Shirish Keskar}
\author[3]{Caiming Xiong}
\author[3]{Richard Socher}
\affil[1]{Department of Statistics, University of California, Berkeley}
\affil[2]{Department of Statistics, Rutgers University - New Brunswick}
\affil[3]{Salesforce Research}
\date{\today}

\usepackage{amssymb,amsmath,multicol,bm,graphics, bbm}
\usepackage[]{geometry}
\usepackage[mathscr]{euscript}
\newcommand{\overbar}[1]{\mkern 1.5mu\overline{\mkern-1.5mu#1\mkern-1.5mu}\mkern 1.5mu}

\newcommand\E{\mathbb{E}}

\newcommand\tj{\tilde{\jmath}}
\newcommand\R{\mathbb{R}}
\newcommand{\cifar}{{\fontfamily{cmtt}\selectfont {cifar10}}}
\newcommand{\mnist}{{\fontfamily{cmtt}\selectfont {mnist}}}
\renewcommand\P{\mathbb{P}}
\usepackage[demo]{graphicx}
\usepackage{subfigure}

\usepackage[english]{babel}
\usepackage{amsfonts}
\usepackage{amsthm}
\usepackage{mathtools,bbold}
\usepackage{graphics}
\usepackage{fancyhdr}
\usepackage{enumerate}

\usepackage[export]{adjustbox}
\usepackage{blindtext}
\usepackage{float}
\usepackage{authblk}
\usepackage{algorithm}
\usepackage[noend]{algpseudocode}
\usepackage{hyperref}
\usepackage{natbib}
\usepackage{tikz}
\newtheorem{theorem}{Theorem}[section]
\newtheorem{definition}{Definition}[section]
\newtheorem{corollary}{Corollary}[theorem]
\newtheorem{lemma}[theorem]{Lemma}
\newtheorem{remark}{Remark}
\begin{document}

\maketitle

\begin{abstract}
 Classical results on the statistical complexity of linear models have commonly identified the norm of the weights $\|w\|$ as a fundamental capacity measure. Generalizations of this measure to the setting of deep networks have been varied, though a frequently identified quantity is the \textit{product of weight norms} of each layer. In this work, we show that for a large class of networks possessing a positive homogeneity property, similar bounds may be obtained instead in terms of the \textit{norm of the product of weights}. Our proof technique generalizes a recently proposed sampling argument, which allows us to demonstrate the existence of sparse approximants of positive homogeneous networks. This yields covering number bounds, which can be converted to generalization bounds for multi-class classification that are comparable to, and in certain cases improve upon, existing results in the literature. Finally, we investigate our sampling procedure empirically, which yields results consistent with our theory. 
\end{abstract}

\section{Introduction}

For classes of linear models, including deep linear networks, it is well known that the norm of the weights $\|w\|$ is an important capacity measure that governs much of their statistical behavior. As a consequence, many algorithms have been developed for these problems that explicitly regularize on such norms. For more complex function classes, such as deep neural networks, various generalizations of this capacity measure have been proposed. Such analyses have commonly identified the product of norms $\prod \|W_\ell\|$ as a complexity measure, e.g. \cite{Bartlett2017, Neyshabur2018ANetworks, NoahGolowichAlexanderRakhlin2018}. In this paper, we show that for a large class of deep networks possessing a positive homogeneity property, including ReLU networks and convolution networks with max or average pooling, we can obtain bounds that are instead in terms of a norm of a product $\|\prod |W_\ell|\|$, which we show often lower bounds the former product of norms. These quantities arise out of a path-based analysis of positive homogeneous networks, and are in fact closely related to the path norms appearing in previous work, such as  \cite{Neyshabur2015Path-SGD:Networks, Neyshabur2015Norm-BasedNetworks, Kawaguchi2019GeneralizationLearning}.

Our method for proving such bounds generalizes a sampling technique recently proposed in \cite{Barron2018ApproximationNetworks}, wherein a positive homogeneous network is parameterized explicitly in terms of a path distribution, which is subsequently sampled from in order to produce a sparse approximant of the original network. This technique allows us to prove that any given positive homogeneous network $f$ admits a sparse approximant $\tilde{f}$ which belongs to a small representer set of functions, whose cardinality we show to be governed by various norms. These results immediately imply covering number bounds, which can be used to control various statistical performance measures, including generalization error.

Our sampling approach is an example of the \textit{probabilistic method}, which, interestingly, appears frequently in previous work on generalization. For example, \cite{Bartlett2017} prove covering number bounds via the use of Maurey sparsification, wherein each layer of a network is sparsified individually using a probabilistic sampling argument. \cite{Arora2018} introduce a compression-based approach to generalization, and prove a bound by compressing each layer of a deep network via a random projection. \cite{Baykal2019} recently introduced an edge-wise sampling procedure for compressing networks, which likewise can be used to obtain generalization bounds via the compression approach of \cite{Arora2018}. Notably, however, in almost all existing works, compression and/or sparsification of deep networks is conducted layer-wise. In the context of norm-based bounds, operating in such a manner generally leads to bounds in terms of the product of norms of the weights, as is the case in \cite{Bartlett2017,Neyshabur2018ANetworks}. 

In contrast, our path-based approach allows us to sample from a distribution over all parameters at once, without the need to work in a layer-wise fashion. More specifically, we define a distribution over paths through the network, obtained by normalizing with various quantities of the form $\|\prod_1^L |W_\ell|\|$. In contrast to the product of weight norms, these quantities capture a notion of \emph{global} variation in networks. Indeed, the product of weight norms measures the size of weights \emph{within} layers, but fails to capture the strength of connections and interactions \emph{across} successive layers. On the other hand, the norm of the product of weights incorporates both aspects. 

We remark that a path-based approach to studying neural networks has appeared in several other works, for example in the design of optimization algorithms \cite{Neyshabur2015Path-SGD:Networks} for neural networks, as well the study of their loss surfaces \cite{Choromanska2015TheNetworks}. The path norms studied in \cite{Neyshabur2015Norm-BasedNetworks} are closely related to the quantities arising in our analysis, and in fact, as we discuss in Section \ref{comparison-generalization}, our results can be seen as improvements of the bounds given therein, by avoiding exponential dependence on network depth. Other path-based capacity measures have been considered as well, notably by  \cite{Kawaguchi2019GeneralizationLearning}, though the resulting bounds depend critically on the (unknown) data distribution. Recent work has also proposed the Fisher-Rao norm as a global, norm-based capacity metric for deep networks, though this has only been shown to control generalization error for the case of linear networks \cite{Liang2019Fisher-RaoNetworks}.

%\vspace{-4mm}
\subsubsection*{Organization and our contributions}
The paper is organized as follows.
\begin{itemize}
    \item In Section 2, we outline our general setting and notation, and review a sampling technique proposed by \cite{Barron2018ApproximationNetworks}, which they use to prove approximation and covering bounds for single output, fully-connected networks in terms of $\ell_1$-type norms. 
    \item In Section 3, we extend the sampling scheme and analysis to the multi-class and convolutional setting, and show that it can be generalized to obtain bounds in terms of a much broader class of norms. We further show with a lower bound that our analysis of the sampling scheme is nearly optimal.
    \item  In Section 4, we exploit certain permutation invariances in deep networks to bound the number of networks that are realized by the sampling method. This results in covering number and metric entropy bounds. We provide several corollaries of these bounds, including notably a new margin-based generalization bound for multi-class classification. We compare this bound to existing results in the literature, and find that our bound is comparable to, and often improves upon, existing norm-based bounds.
    
    \item In Section 5, we investigate empirically the sampling strategy studied theoretically above, and find that compressibility of networks correlates well with generalization performance. An analysis of certain normalized margin distributions suggests that the quantities appearing in our bounds do indeed capture this behavior.
    
    \item Finally, in Section 6, we suggest directions for future research; namely, we outline one potential approach to extending our technique and results to the analysis of residual networks.
\end{itemize}

\section{Setup and Background}
\label{section:setup}
In this work, we consider a standard setting of multi-class classification, wherein a network $f(x;W) : \mathcal{X}\subseteq \R^d \mapsto \R^k$ makes a classification decision $\hat{y} = \arg\max_j f(x;W)_j$. We use $S$ to denote a training set $\{x^1,...,x^n\}$ of $n$ points, each of which has a corresponding label $\{y^1,...,y^n\}$, with $y^i \in \{1,...,k\}$. For a vector $v\in 
\R^k$ and $y\in \{1,...,k\}$, we define the margin operator to be $\mathcal{M}(v, y) = v_y - \max_{j\neq y} v_j$. We denote the classification loss 
\begin{align}
\ell(f(x;W), y) = \mathbbm{1}\big(\mathcal{M}(f(x;W),y) \leq  0\big)
\end{align}
and, for $\gamma >0$, the $\gamma$-margin loss
\begin{align}
 \ell_\gamma(f(x;W), y) = \mathbbm{1}\big(\mathcal{M}(f(x;W),y) \leq  \gamma \big).   
\end{align}
For the empirical loss, we denote $\hat{\ell}(f) = \frac{1}{n}\sum_{(x,y) : x\in S} \ell(f(x;W),y)$ and for the population loss we use $\ell(f) = \E_{(x,y)}[\ell(f(x;W),y)]$, and likewise for $\hat{\ell}_\gamma(f)$ and $\ell_\gamma(f)$. We will also use the ramp function which, for any $\gamma >0$, is given by
\begin{align}
    R_\gamma(z) = \begin{cases}
    0 & z < -\gamma,\\
    1+z/\gamma & z\in [-\gamma,0]\\
    1 &z>0
    \end{cases}.
\end{align}
$R_\gamma$ importantly satisfies the following:
\begin{align}
    \ell(f) \leq \E_{(x,y)}[R_\gamma(-\mathcal{M}(f(x;W),y)] \leq \ell_\gamma(f)
\end{align}
with analogous inequalities holding for $\hat{\ell}(f),\hat{\ell}_\gamma(f)$ with the empirical distribution on $S$.

We will use the notation $\mathbb{B}_q(r) = \{x\in \R^d : \|x\|_q \leq r\}$ to denote the $\ell_q$ balls in $\R^d$. For an $m\times n$ matrix $A$, we define the matrix $q$-norm induced by the $\ell_q$ norm by
\begin{align}
    \|A\|_q = \sup_{z\neq 0} \frac{\|Az\|_q}{\|z\|_q}.
\end{align}
We will be particularly interested in $\|\cdot\|_2$, which is the spectral norm (also denoted $\|\cdot\|_\sigma$), and $\|\cdot\|_\infty$, which is also equal to the $(1,\infty)$ norm appearing in the analysis of \cite{NoahGolowichAlexanderRakhlin2018}. We will also use the matrix $(q,1)$ norm, which is given by
\begin{align}
    \|A\|_{q,1} = \sum_{j=1}^n\Big(\sum_{i=1}^m|a_{ij}|^q\Big)^{1/q}.
\end{align} 

We consider networks of the form
\begin{align}
    f(x;W) = W_L\phi(W_{L-1}\phi(\cdots\phi(W_1x))) 
    % = 
    % \sum_{j_L}w_{j_L}\phi\Big(\sum_{j_{L-1}} w_{j_L, j_{L-1}}\phi\Big(\cdots \phi\Big(\sum_{j_1}w_{j_2,j_1}x_{j_1}\Big)\Big)\Big)
\end{align}
where $ W_{\ell}[j_{\ell}, j_{\ell-1}] = w_{j_{\ell},j_{\ell-1}} $ are the $ d_{\ell} \times d_{\ell-1} $ weight matrices ($d_0=d$ and $d_L=k$) and $\phi(z)$ is an activation which is positive homogeneous\footnote{The property that for all $\alpha > 0$, $\phi(\alpha z) = \alpha \phi(z)$.}, 1-Lipschitz and satisfies $\phi(0) = 0$. The most common example is the ReLU activation $\phi(z) = \max(z,0)$, though it also applies to other activations, such as the `leaky ReLU' or the identity. Since compositions of positive homogeneous functions are also positive homogeneous, our theory applies to max or average pooling operations followed by a positive homogeneous activation, as in convolutional networks.
For each output unit $j_L \in \{1,...,k\}$, the subnetwork terminating at node $j_L$ may be expressed as
\begin{align}
     f(x;W)_{j_L}= \sum_{j_{L-1}}w_{j_L,j_{L-1}}\phi\Big(\cdots \phi\Big(\sum_{j_0}w_{j_1,j_0}x_{j_0}\Big)\Big).
\end{align}
Each unit of the network outputs $ x_{j_{\ell}} = \phi(z_{j_{\ell}}) $ for a corresponding input 
\begin{align}
z_{j_{\ell}} = \sum_{j_{\ell-1}}w_{j_{\ell},j_{\ell-1}}x_{j_{\ell-1}} = \sum_{j_{\ell-1}}w_{j_{\ell},j_{\ell-1}}\phi(z_{j_{\ell-1}}).
\end{align}

\subsection{$\ell_1$ normalization and the path distribution}
A crucial observation made in \cite{Barron2018ApproximationNetworks} is that by doubling the number of nodes per layer and relabeling the indices, we can assign the absolute weights $ |w_{j_{\ell},j_{\ell-1}}| $ into one of two pre-specified groups, each of size $ d_{\ell-1} $: (I) if $ w_{j_{\ell},j_{\ell-1}} $ is negative, we associate it with $ -\phi(z_{j_{\ell-1}}) $ and (II) if $ w_{j_{\ell},j_{\ell-1}} $ is positive, we associate it with $ +\phi(z_{j_{\ell-1}}) $. By doing this, we can assume all the weights are nonnegative. For notational convenience, we do not explicitly account for these sign differences in the activation function when we describe the network. Instead, without loss of generality, we simply write $ \phi $ with the understanding that it is a placeholder for either $ -\phi $ or $ +\phi $. Likewise, we use $W_\ell$ with the understanding that the weights are taken to be nonnegative, which results in quantities that are in terms of the absolute values of the original weights.

With this convention, we may exploit the positive homogeneity of $\phi$, and move all the (non-negative) weights to the inner layer sum to get  
\begin{align}
     f(x;W)_{j_L} = \sum_{j_{L-1}}\phi\Big(\sum_{j_{L-2}}\phi\Big(\cdots \phi\Big(\sum_{j_0}w_{j_0,j_1,...,j_L}x_{j_0}\Big)\Big)\Big) 
     \label{eqn:path-param}
\end{align}
where 
\begin{align}
    w_{j_0,j_1,...,j_L} = w_{j_L,j_{L-1}}w_{j_{L-1},j_{L-2}}w_{j_{L-2},j_{L-3}}\cdots w_{j_1,j_0}.
\end{align}
Here we think of each $(j_0,...,j_L)$ as indexing a single \textit{path} through the network. It is this representation of the network in terms of the paths that facilitates our analysis. 

\begin{remark}[Path representation for networks with pooling]
We remark that a similar expression can be used to study convolutional networks with max and/or average pooling, by recalling that 2D convolution can be expressed as matrix multiplication with respect to a particular class of Toeplitz matrices and that the (max or average) pooling operator $\mathcal{P}$ is positive homogeneous. For simplicity, in what follows we omit the use of $\mathcal{P}$, hence considering feed-forward networks or convolution networks without pooling, though we address the details of the convolutional case with pooling in Appendix \ref{pooling-details}.
\end{remark}

We now turn our attention to \textit{normalizing} the path weights $(w_{j_0,j_1,...,j_L})_{j_0,j_1,...,j_L}$ in such a way that they may form a probability distribution. Since the $w_{j_0,...,j_L}$ are non-negative, the simplest way to do this would be to normalize by their sum, which is also equal to the $1$-norm of the product of the (non-negative) weights:
\begin{align}
    \mathscr{V}_1= \sum_{j_0,j_1,...,j_L} w_{j_0,j_1,...,j_L} = \Big\|\prod_1^L |W_\ell|\Big\|_{1,1}.
\end{align}

Now by construction we see that we can equally well express the function as
\begin{align}
     f(x;W)_{j_L} = \mathscr{V}_1\sum_{j_{L-1}}\phi\Big(\sum_{j_{L-2}}\phi\Big(\cdots \phi\Big(\sum_{j_0}p_{j_0,j_1,...,j_L}x_{j_0}\Big)\Big)\Big), 
    \label{pathrepresentation}
\end{align}
where $p_{j_0,...,j_L} = \frac{1}{\mathscr{V}_1}w_{j_0,...,j_L}$. We see that by design, $p_{j_0,...,j_L} \geq 0$ and $\sum_{j_0,...,j_L} p_{j_0,...,j_L} = 1$. Hence we can view $(p_{j_0,...,j_L})_{j_0,...,j_L}$ as a discrete distribution over the multi-indices $(j_0,...,j_L)$, which we interpret as a path (a sequence of nodes) through the network. We call $(p_{j_0},...,p_{j_L})_{j_0,...,j_L}$ the 1-\emph{path distribution} and the normalizing factor $\mathscr{V}_1$ the 1-\emph{path variation} of the network $f(x;W)$. As we discuss in Section \ref{comparison-generalization}, in the single output case, $\mathscr{V}_1$ is in fact the same as the $1$-path norm, studied in \cite{Neyshabur2015Norm-BasedNetworks}. 

Another quantity that will arise in our analysis is related to the 1/2-\emph{Renyi entropy} of the marginal distributions $p_\ell$ (obtained by marginalization of the path distribution $p_{j_0,j_1,...,j_\ell}$). We define the 1-path complexity to be 
\begin{align}
    \zeta_1 = \frac{1}{L}\Big(1+\sum_{\ell=1}^{L-1} e^{\frac{1}{2}H_{1/2}(p_\ell)}\Big)
\end{align}
%{\color{red}{Not sure where the summation begins and ends...I'll have to think more about that.}}
Since $0\leq H_{1/2}(p_\ell) \leq \log(d_\ell)$, we have $1\leq \zeta_1\leq \frac{1}{L}(1+\sum_{\ell=1}^{L-1} \sqrt{d_\ell})$, though this quantity can be substantially smaller when the marginal distributions $ p_{\ell} $ are non-uniform over units in the network. Hence $\zeta_1$ can be thought of as a measure of the average \emph{effective} square-root width of the intermediate layers.

Importantly, the path distribution $p$ can be shown to possess a Markov structure (see \cite{Barron2018ApproximationNetworks}), allowing us to write 
\begin{align}
     p_{j_0,...,j_L} = p_{j_L}p_{j_{L-1}|j_L}p_{j_{L-2}|j_{L-1}}\cdots p_{j_0|j_1}
     \label{markovstructure}
\end{align}
and the network correspondingly as
\begin{align}
   \mathscr{V}_1f(x;P) = \mathscr{V}_1P_L\phi(P_{L-1}\phi(\cdots\phi(P_1x)))
   \label{conditional-representation}
\end{align}
where $ P_{\ell} $ is a transition matrix for the Markov distribution $p$, $P_\ell[j_{\ell},j_{\ell-1}] = p_{j_{\ell-1}|j_{\ell}}$ for $\ell<L$ and $P_L[j_L,j_{L-1}] = p_{j_L,j_{L-1}} = p_{j_L}p_{j_{L-1}|j_L}$. 
% As we will see, it is this representation that facilitates our analysis, and enables us to obtain sparsification and generalization guarantees.

% Each conditional probability $ p_{j_\ell|j_{\ell+1}}$ can be easily computed from the collection of successive matrix products $ \{W_{\ell}W_{\ell-1}\cdots W_1:\ell=1,2,\dots, L\} $ (the collection itself can be inductively constructed), since
% $$
% p_{j_\ell|j_{\ell+1}} = w_{j_{\ell+1},j_{\ell}}\frac{\|W_{\ell}[j_{\ell},]W_{\ell-1}\cdots W_1\|_1}{\|W_{\ell+1}[j_{\ell+1},]W_{\ell}\cdots W_1\|_1}
% $$
% and
% $$
% p_{j_L} = \frac{\|W_{L}[j_{L},]W_{L-1}\cdots W_1\|_1}{V},
% $$
% where $ W_{\ell}[j_{\ell},] $ (resp. $ W_{\ell}[,j_{\ell-1}] $) is row (resp. column) $ j_{\ell} $ (resp. $ j_{\ell-1}$) of $ W_{\ell} $. Thus, the conditional probabilities are reweighted versions of the weight matrices.

% The marginal probabilities are equal to
% $$
% p_{j_{\ell}} = \frac{\|W_L\cdots W_{\ell+2}W_{\ell+1}[,j_{\ell}]\|_1\|W_{\ell}[j_{\ell},]W_{\ell-1}\cdots W_1\|_1}{V}.
% $$

%\section{Path Sampling and Sparse Approximation}
%\subsection{Background}
\subsection{Constructing sparse approximants from the path distribution}
The representations (\ref{pathrepresentation}) suggests an approach for constructing an approximant $\tilde{f}$ of $f$, by taking $\tilde{f} = f(x;\tilde{p})$ for some estimate $\tilde{p}$ of $p$. Since $p$ is a probability distribution, a natural candidate for an approximant $\tilde{p}$ is an empirical distribution which arises from taking $M$ independent samples from the path distribution $p$. We refer to such an empirical distribution as $\tilde{p}_M$, or simply $\tilde{p}$, when the number of samples $M$ is clear. If one can then bound $\E_{\tilde{p}}[\|f(x;p) - f(x;\tilde{p})\|] \leq \delta_M$,
for some $\delta_M$, then since the average over $\tilde{p}$ is always more than the minimum over $\tilde{p}$, one can deduce the existence of some $\tilde{p}$ for which $\|f(x;p) - f(x;\tilde{p})\| \leq \delta_M$. This type of reasoning is known as the \emph{probabilistic method}, and appears in many results in the literature, as we discuss in Section \ref{comparison-sparsification}. It is also employed in the main result of \cite{Barron2018ApproximationNetworks}, which we now review.

Consider sampling $K = (K_{j_0,j_1,...,j_L})_{j_0,j_1,...,j_L} \sim \text{Multinomial}(M,p)$, where $K_{j_0,j_1,...,j_L}$ is the number of times the path $(j_0,j_1,...,j_L)$ appeared in the $M$ samples. One could then take an approximant to be $\tilde{p}=K/M$. However, this $\tilde{p}$ would not necessarily factor into matrices, which is favorable both for practical reasons and for the sake of analysis. Instead, \cite{Barron2018ApproximationNetworks} construct $\tilde{p}_{j_0,j_1,\dots,j_L} = \tilde p_{j_L} \tilde p_{j_{L-1}|j_{L}} \cdots \tilde p_{j_0|j_1} $ as the empirical Markov distribution on the paths $(j_0,j_1,...,j_L)$, where 
\begin{align}
    \tilde{p}_{j_\ell} = \frac{K_{j_\ell}}{M}, \hspace{5mm} \tilde{p}_{j_\ell, j_{\ell+1}} = \frac{K_{j_\ell, j_{\ell+1}}}{M}, \hspace{5mm} \tilde{p}_{j_{\ell}|j_{\ell+1}} = \frac{\tilde{p}_{j_\ell, j_{\ell+1}}}{\tilde{p}_{j_{\ell+1}}}
\end{align}
with the convention that 0/0 = 0. Here $K_{j_\ell}, K_{j_\ell,j_{\ell+1}}$ are the marginal and pairwise counts, respectively, obtained by summing out $K_{j_0,..,j_\ell, j_{\ell+1},..,j_L}$ over unspecified indices. 

Another more principled reason to favor a network built from the above quantities is that, within the class of Markov distributions, $ \tilde p $ is the (restricted) maximum likelihood estimator (MLE) of $ p $ from the empirical counts $ K $. We state this formally in our first theorem. At a high-level, it says that, among plug-in approximants of the original network, the one using the empirical Markov distribution is `optimal'.
\begin{theorem}
$$
\tilde p = \underset{p \;\text{Markov}}{\arg\max}\; \mathcal{L}(p),
%\bigg\{
%M!\prod_{(j_0,j_1,\dots,j_L)}\frac{ %p_{j_0,j_1,\dots,j_L}^{K_{j_0,j_1,\dots,j_L}}}{K_{j_0,j_1,\dots,j_L}!}
%\bigg\},
$$
where $ \mathcal{L}(p) = M!\prod_{(j_0,j_1,\dots,j_L)}\frac{ p_{j_0,j_1,\dots,j_L}^{K_{j_0,j_1,\dots,j_L}}}{K_{j_0,j_1,\dots,j_L}!} $ is the likelihood of the count vector $ K $.
\end{theorem}
\begin{proof}
This can be shown by combining the fact that the (unrestricted) MLE of a multinomial distribution is the empirical class proportion vector $ K/M $, together with the invariance property of MLEs.
\end{proof}

Throughout the paper, we think of the number $M$ as a parameter which controls the level of compression of $f(x;\tilde{p}_M)$ relative to $f(x;p)$. Intuitively, it is clear that as $M$ gets large, $f(x;\tilde{p}_M)$ more closely approximates $f(x;p)$. Moreover, $M$ controls the sparsity and precision of the parameters $\tilde{p}_M$, which is demonstrated by the following facts: first, the number of nonzero parameters $\tilde{p}_M$ is upper bounded deterministically by $LM$, and, second, the (base-10) precision of the $\tilde{p}_M$ is upper bounded by $\log_{10}(M)$. Hence, we can think of $\tilde{p}_M$ as also a natural quantization of the weights $p$.

%%% commenting this out to see how it affects space %%%
% \footnote{In fact, there exists a compressed network with at most $ \sum_{\ell} d_{\ell}d_{\ell-1}(1-(1-d^{-1}_{\ell}d^{-1}_{\ell-1})^M) $ parameters, which is slightly smaller than $ LM $ and less than the total number of parameters in the original network, i.e., $ \sum_{\ell} d_{\ell}d_{\ell-1} $. If all the intermediate layer dimensions are equal to $ h $, the sparsity of the original network is reduced by a factor of $ 1-(1-h^{-2})^M $.

In the single output case, \cite{Barron2018ApproximationNetworks} prove the following $L_2$ bound (adapted slightly to match our notation), when $\mathcal{X} = [-1,1]^d$.
\begin{theorem}[\cite{Barron2018ApproximationNetworks}, Theorem 1]
Let $f(x;W)$ be a single output ReLU network with 1-path variation $\mathscr{V}_1$ and 1-path complexity $\zeta_1$, and let $\P$ be probability measure on $[-1,1]^d$. Then
\begin{align}
    \E_{\tilde{p}}\Big[\int |f(x;W) - f(x;\widetilde{W})|^2 \P(dx)\Big] \leq \Big(\frac{\mathscr{V}_1\zeta_1L}{\sqrt{M}}\Big)^2,
\end{align}
where $f(x;\widetilde{W}) = \mathscr{V}_1f(x;\tilde{p})$. 
\label{thm:barron-theorem}
\end{theorem}

In the next section, we extend this result in the following ways: first, we show that we can obtain path distributions by normalizing by a much broader class of path variations than the 1-path variation $\mathscr{V}_1$. We also extend the bound to the multi-output setting, where we obtain a bound on the $\ell_2$ norm of outputs. This result allows us to later study multi-class classification. Finally, we show that similar results can also be obtained for networks with pooling layers, though we defer the details of this case to the Appendix.

\section{Path Sampling and Sparse Approximation}
\subsection{Path sampling with general norms}
In this section, we show how the sampling scheme summarized in the previous section can be generalized to norms besides $\|\cdot\|_{1,1}$. To see how this is possible, notice that for any $w_{j_0}$, we can express $f(x;W)_{j_L}$ as
\begin{align}
    \sum_{j_{L-1}}\phi\Big(\sum_{j_{L-2}}\phi\Big(\cdots \phi\Big(\sum_{j_0}w_{j_0}w_{j_0,j_1,...,j_L}x'_{j_0}\Big)\Big)\Big)
\end{align}
where $x'_{j_0} = x_{j_0}/w_{j_0}$. Now for 
\begin{align}
\mathscr{V} = \sum_{j_0,...,j_L}w_{j_0}w_{j_0,...,j_L},
\label{eqn:bound-requirement}
\end{align}
we can define a path distribution $p_{j_0,...,j_L} = \frac{1}{\mathscr{V}}w_{j_0}w_{j_0,...,j_L}$.

Now let $1\leq q \leq \infty$ and let $q^*$ be its conjugate exponent (so that $\frac{1}{q}+\frac{1}{q^*} = 1$). For a dataset $S$, we consider
\begin{align}
    w_{j_0}^{(q)} = \begin{cases}
    (n^{-1}\sum_{x\in S} |x_{j_0}|^{q^*})^{1/q^*} & 1< q \leq \infty \\
    \max_{x\in S} |x_{j_0}| & q=1
    \end{cases}
\end{align}
which gives rise to the \textit{$q$-path variation}
\begin{align}
    \mathscr{V}_q = \sum_{j_0,...,j_L}w_{j_0}^{(q)}w_{j_0,...,j_L}
\end{align}
The value in these definitions is captured in the following lemma, which shows that we can bound $\mathscr{V}_q$ in terms of norms $\|\prod_1^L |W_\ell|\|$.

\begin{lemma}
Let $1\leq q \leq \infty$, and let $q^*$ be its conjugate exponent. Then 
\begin{align}
\mathscr{V}_q \leq (\max_{x\in S} \|x\|_{q^*})k^{1-1/q^*}\Big\|\prod_1^L |W_\ell|\Big\|_{q^*}
\end{align}
and 
\begin{align}
   \mathscr{V}_q \leq (\max_{x\in S} \|x\|_{q^*})\Big\|\prod_1^L |W_\ell|\Big\|_{q,1} .
\end{align}
\label{thm:variation-bounds}
\end{lemma}
\begin{proof}
The proof is several simple applications of H\"older's inequality. See \ref{bounding-normalizing-constants} for details. 
\end{proof}

With the above in mind, we introduce the following definitions.
\begin{definition}
For $1\leq q \leq\infty$, we define the $q$-path distribution by
\begin{align}
    p^{(q)}_{j_0,...,j_L} = \frac{w_{j_0}^{(q)}w_{j_0,...,j_L}}{\mathscr{V}_q}.
\end{align}
We define the $q$-path complexity by
\begin{align}
    \zeta_q = \frac{1}{L}\Big(1 +\sum_{\ell = 1}^{L-1}e^{\frac{1}{2} H_{1/2}(p_\ell^{(q)})}\Big).
\end{align}
\end{definition}

Notice that for $q=1, r=1$, the above definitions reduce to the setting of the Section \ref{section:setup}, with $p^{(1)}$ obtained by normalizing by $\|\prod_1^L|W_\ell|\|_{1,1}$ and $x\in [-1,1]^d$. In the next section, we use the path distributions $p^{(q)}$ to obtain sparsification results for vector-valued neural networks.

\subsection{Bounds for Path Sampling with Deep ReLU Networks}
In this section, we extend the analysis of Theorem \ref{thm:barron-theorem} to the multi-output and convolution setting, and show that similar bounds may be obtained in terms of $\mathscr{V}_q$, for $q\in [1,2]$.  

\begin{theorem}
Let $f(x;W)$ be an $L$-layer ReLU network, $S$ a dataset, and let $1 \leq q\leq 2$. If $\tilde{p}$ is the Markov distribution formed from $M$ samples from $p^{(q)}_{j_0,j_1,...,j_L}$, then 
    \begin{align}
        \E_{\tilde{p}}\Big[\frac{1}{n}\sum_{x\in S}\|f(x;\widetilde{W}) - f(x;W)\|_2^2 \Big]
    \leq \Big(\frac{\mathscr{V}_{q}\zeta_q L}{\sqrt{M}}\Big)^2,
    \end{align}
where $f(x;\widetilde{W}) = \mathscr{V}_qf(x;\tilde{p})$. 
\label{thm:maximal-bound}
\end{theorem}

\begin{proof}
See \ref{thm:maximal-bound-appendix}.
\end{proof}

Using Lemma \ref{thm:variation-bounds}, Theorem \ref{thm:maximal-bound} can be used to give bounds, for example, in terms of the matrix $(2,1)$ norm, the spectral norm, and the $(1,\infty)$ norm. 

Since the minimum over $\tilde{p}$ is always less than the expected value, the above results imply, for example, the existence of representer $f(x;\widetilde{W}) = \mathscr{V}_qf(x;\tilde{p}_M)$ such that 
\begin{align}
    \sqrt{\frac{1}{n}\sum_{x\in S}\|f(x;\widetilde{W}) - f(x;W)\|_2^2} \leq \frac{\mathscr{V}_q\zeta_q L}{\sqrt{M}}.
\end{align}

It turns out that, in the case of $q=1$, the error analysis in Theorem \ref{thm:maximal-bound} is optimal for single output, two layer networks.

\begin{theorem}
There exists a dataset $S\subseteq [-1,1]^d$, a single output, two layer network $f(x;W)$, and an integer $M_0$ such that for all $M\geq M_0$, we have
\begin{align}
    \E_{\tilde{p}}\Big[\frac{1}{n}\sum_{x\in S} |f(x;\widetilde{W}) - f(x;W)|^2\Big] = \Omega\Big(\frac{\mathscr{V}_1\zeta_1}{\sqrt{M}}\Big)^2.
\end{align}
\end{theorem}
\begin{proof}
See \ref{sampling-lb-appendix}.
\end{proof}

% While the above proves the existence of a sparsification $\tilde{f}$ for a given classifier $f$, we can easily translate this into an algorithm which returns an accurate sparsification with high probability.

% \begin{algorithm}
%  $\mathbf{input}$ $f(x;W)$, $S$, $M$, $\delta$\\
%  let $T = M\log(1/\delta)$\;
%  \For{t \in [T]}{
%   sample $K^t \sim \text{ Mult}(M,p)$, form $\tilde{p}^t$\;
%   compute $\mathsf{err}_t := \max_{x\in S} \|f(x;p) - f(x;\tilde{p}^t)\|_\infty$\;
%  }
%  set $t_\star := \arg\min_{t\in [T]} \mathsf{err}_t$\;
%  $\mathbf{return}$ $f(x;\widetilde{W}) = Vf(x;\tilde{p}^{t_\star})$
%  \caption{ReLU Path Sampling}
%  \label{alg:algorithm1}
% \end{algorithm}

% \begin{corollary}
% Let $f(x;W)$ be a ReLU network, $\delta \in (0,1)$ and $M\geq 1$. Then with probability at least $1-\delta$, Algorithm \ref{alg:algorithm1} returns a classifier $\tilde{f} = f(x;\widetilde{W})$ such that for all $x\in S$,
% \begin{align}
%     \|f(x;W) - f(x;\widetilde{W}\|_\infty \leq \frac{8V(f)\zeta(p)L\sqrt{\log(nk)}}{\sqrt{M}}
% \end{align}
%  \end{corollary}
 
%  We remark that a reasonable concern regarding the practical implementation of Algorithm \ref{alg:algorithm1} involves computing and storing the path distribution $p$ and obtaining samples from Multinomial$(M,p)$. We explain in Section \ref{empirical-investigation} how to resolve these concerns.

\subsection{Comparison to existing techniques}
\label{comparison-sparsification}
 It is worth taking a moment to compare our technique and results to existing work. The probabilistic method, interestingly, appears frequently. 

In \cite{Bartlett2017}, layers are sparsified individually, using a technique known as \emph{Maurey sparsification}. This type of reasoning in the context of function approximation is due to \cite{Pisier1980RemarquesMaurey} and was later applied to single output, single hidden layer networks (i.e., $ k = 1 $ and $ L = 2 $) in the seminal work of \cite{Barron1991ComplexityNetworks, Barron1993UniversalFunction}. \cite{Bartlett2017} use a generalization of this technique given in \cite{Zhang2002CoveringClasses} to establish the existence of an approximant $\widetilde{U}$ of a matrix $U$ by defining a distribution over $\widetilde{U}$ and bounding $\E\|\widetilde{U} - U\|^2$, though using this approach to analyze multilayer networks results in error bounds that scale with $\prod_1^L \|W_\ell\|$, which arises from a worst case analysis of error propagation between layers. In contrast, our technique takes advantage of \emph{global} structure in the network, and as a consequence instead scales with a quantity $\|\prod_1^L |W_\ell|\|$.

% Several other matrix sparsification results likewise employ various sampling techniques, and could also be used to obtain full network compression by conducting layer-wise sparsification (see for example \cite{Achlioptas2013MatrixBest,Drineas2011AInequality,  Kundu2014ASparsification}). These involve sampling an entry $(i,j)$ of a matrix $W$ with probability $p_{ij}^{(1)} = \frac{|w_{ij}|}{\|W\|_1}$, $p_{ij}^{(F)} = \frac{w_{ij}^2}{\|W\|_F^2}$ and $p_{ij}^{(1,F)} = \frac{1}{2}[p_{ij}^{(1)} + p_{ij}^{(F)}]$, respectively. An analysis similar to \cite{Bartlett2017} could likely show similar norm-based controls on error for deep networks, though this would again be somewhat na\"ive, as it fails to exploit additional structure distinguishing neural networks from any arbitrary sequence of matrices. 

% The benefit of our approximation results is that, rather than treating a network as an arbitrary sequence of matrices, we exploit additional structure provided by the functional form of deep ReLU networks. This manifests most importantly in the Markov structure, which is critical to our proof technique.   

Other examples of sampling techniques include the recent work of \cite{Arora2018} and \cite{Baykal2019}. The former uses a Johnson-Lindenstrauss-type random projection to compress each layer of a deep network, which deduces the existence of a compression by showing that the probability of sampling an approximant at the desired level of accuracy is nonzero. \cite{Baykal2019} instead use an edge-wise sampling approach which is similar to existing matrix sparsification techniques (e.g. \cite{Achlioptas2013MatrixBest, Kundu2014ASparsification, Drineas2011AInequality}), but notably improves these methods by using a held-out set of data to measure the sensitivity of each layer's output to certain edges. In both cases, however, the error analysis does not involve norm quantities, and instead involves strongly data dependent quantities which are harder to compute and interpret. As a consequence of the stronger dependence on the function and dataset $S$, these techniques can only be used to prove a slightly weaker notion of generalization; they prove only the generalization of the compressed network, rather than the original network. It is nonetheless a fascinating direction for future research to study if stronger data-dependence may be incorporated into our path-based approach to obtain better error bounds. 

% It is worth noting that there also exist techniques that do not appeal to the probabilistic method. In particular, \cite{Arora2018} show that the approximation of \cite{Neyshabur2018ANetworks} can be obtained via a deterministic compression. In that case, layers were sparsified by a truncation of singular values. More generally in the model compression literature, deterministic compressions are obtained through various pruning and quantization techniques \cite{Cheng2018ModelChallenges}, though without the same theoretical guarantees as the above methods. 

\section{Covering, Metric Entropy, and Implications for Generalization}
In this section, we show that the sampling results given in the previous section can be used to obtain covering number bounds, which imply generalization bounds for multi-class classification. The approach is similar to that used to prove the results of \cite{Bartlett2017}. Throughout this section, we use $\mathcal{F}(\mathscr{V}_q,\zeta_q)$ to denote the class of positive homogeneous networks with $q$-path variation at most $\mathscr{V}_q$ and $q$-path complexity at most $\zeta_q$, and for any $\gamma > 0$, we denote $\mathcal{F}_\gamma(\mathscr{V}_q,\zeta_q) = R_\gamma \circ (-\mathcal{M}) \circ \mathcal{F}(\mathscr{V}_q,\zeta_q)$. 

We first recall a few definitions.

\begin{definition}
For a class of real-valued functions $\mathcal{F}$ and $\epsilon > 0$, a set $\mathcal{G}_\epsilon$ is an \emph{$\epsilon$-covering} of $\mathcal{F}$ (with respect to $\|\cdot\|_{2,S}$) if for all $f\in\mathcal{F}$, there exists $g\in \mathcal{G}_\epsilon$ such that 
$
    \|f-g\|_{2,S} = \sqrt{\frac{1}{n}\sum_{x\in S}|f(x) - g(x)|^2} \leq \epsilon.
$
We define the \emph{$\epsilon$-covering number} $\mathcal{N}_2(\epsilon, \mathcal{F}, S)$ to be the minimum cardinality of covering sets $\mathcal{G}_\epsilon$. Finally, the \emph{$\epsilon$-metric entropy} of $\mathcal{F}$ is defined to be $
\log \mathcal{N}_2(\epsilon,\mathcal{F},S).
$
\end{definition}

To use the sampling bounds from Theorem \ref{thm:maximal-bound} to get covering number bounds, we need to bound the cardinality of the set of functions $\tilde{f}$ arising from $M$ samples. The below gives such a bound. 

\begin{theorem}
The number of networks $f(x; \tilde p)$ that arise from the sampling scheme is at most $ 8^{ML}(d e)^M $. Thus, the log-cardinality of the representor set is bounded by $ M(\log(d e)+L\log (8)) $.
\label{thm:cardinality-bound}
\end{theorem}
\begin{proof}
See \ref{thm:cardinality-appendix}.
\end{proof}

We remark that a more na\"ive bound may be obtained by simply counting the total number of possible samples $K=K_M$ that can arise from sampling Multinomial$(M,p)$. This can be shown by a standard combinatorial argument to be $\binom{M + D - 1}{M}$, where $D = d_0d_1\cdots d_L$. Theorem \ref{thm:cardinality-bound} improves this bound considerably by recognizing that there are many samples $K_M$ which result in the same function $f(x;\tilde{p}_M)$. Exploiting this observation, the proof takes advantage of the permutation invariance of units in deep networks, which implies that the number of functions that can be obtained from the sampling scheme depends only on the number of ways we can partition the integer $M$ into pairwise counts $K_{j_\ell, j_{\ell+1}}$. As a consequence, the cardinality is \emph{completely} independent of the intermediate layer dimensions $ d_{\ell} $ for $ \ell = 1, 2, \dots, L $, except for mild logarithmic dependence on the input dimension $ d $. It turns out that in the setting of the $2$-path variation, we can use a trick from \cite{Zhang2002CoveringClasses} to remove dependence on $d$ altogether, though we defer the details of this to Appendix \ref{thm:effictive-dim-appendix}.

Using the fact that $\mathcal{M}$ is 2-Lipschitz with respect to $\|\cdot\|_2$ (see appendix of \cite{Bartlett2017} for details), and $R_\gamma$ is $\frac{1}{\gamma}$ Lipschitz, we may use this result together with Theorem \ref{thm:maximal-bound} to obtain the following metric entropy bounds.

\begin{corollary}
\label{thm:covering-number-mt}
Let $\epsilon,\gamma >0$, $1\leq q \leq 2$. Then 
\begin{align}
    \log\mathcal{N}_2(\epsilon,\mathcal{F}_\gamma(\mathscr{V}_q,\zeta_q),S) \leq \frac{9\mathscr{V}_q^2\zeta_q^2L^2(L+\log(de))}{\gamma^2\epsilon^2}
    \end{align}
\end{corollary}
Using standard techniques, Corollary \ref{thm:covering-number-mt} implies the following generalization guarantees.
\begin{theorem}
Let $f(x;W)$ be an $L$-layer positive homogeneous network and let $\delta \in (0,1)$. 
For any $1\leq q \leq 2$ and $\gamma >0$, with probability at least $1-\delta$ over the training set $S$, the generalization error $\ell(f) - \hat{\ell}_\gamma(f)$ is bounded by
    \begin{align}
        \widetilde{\mathcal{O}}\Big(\frac{ \mathscr{V}_q\zeta_qL\sqrt{L+\log(d)}}{\gamma \sqrt{n}}+\sqrt{\frac{\log(1/\delta)}{n}} \Big),
    \end{align}
where $\mathscr{V}_q,\zeta_q$ are the $q$- path variation and path complexity of $f$. 
\label{thm:gen-bound-mt}
\end{theorem}
\begin{proof}
See \ref{thm:gen-bound-appendix}.
\end{proof}

Plugging in the bounds from Lemma \ref{thm:variation-bounds}, this implies a host of norm-based generalization bounds, summarized in the following Corollary.

\begin{corollary}
Let $f(x;W)$ be an $L$-layer positive homogeneous network and let $\delta \in (0,1)$. 
For any $1\leq q \leq 2$ and $\gamma >0$, with probability at least $1-\delta$ over the training set $S$, the generalization error $\ell(f) - \hat{\ell}_\gamma(f)$ is bounded by
\begin{align}
   \widetilde{\mathcal{O}}\Big(\frac{\max_{x\in S} \|x\|_{q^*} \big\|\prod_1^L |W_\ell|\big\|_{q,1}\zeta_q L \sqrt{L+\log(d)}}{\gamma\sqrt{n}} +\sqrt{\frac{\log(1/\delta)}{n}} \Big), 
\end{align}
as well as
 \begin{align}
     \widetilde{\mathcal{O}}\Big(\frac{\max_{x\in S} \|x\|_{q^*}k^{1-1/q^*} \big\|\prod_1^L |W_\ell|\big\|_{q^*}\zeta_q L \sqrt{L+\log(d)}}{\gamma\sqrt{n}} +\sqrt{\frac{\log(1/\delta)}{n}} \Big).
 \end{align}
 \label{thm:norm-based-bounds}
\end{corollary}

\subsection{Comparison to Existing Generalization Bounds}
\label{comparison-generalization}

Before comparing the bounds from Theorem \ref{thm:gen-bound-mt} to existing norm-based bounds, we remark that by arranging for the weights to be positive, the matrix products we obtain above are in terms of \emph{absolute values} of the original weight matrices. However, we can still compare our bounds to those that use products of matrix norms. For example, for entry-wise norms such as $(q,1)$ norms, the $(1,\infty)$ norm, and the Frobenius norm, it is always the case that $ \|\prod_{\ell}|W_{\ell}|\| \leq \prod_{\ell}\|W_{\ell}\| $. On the other hand, note that $ \||A|\|_{\sigma} \geq \|A\|_{\sigma} $ and so the term $ \|\prod_1^L |W_\ell|\|_\sigma $ is not directly comparable to $ \prod_{\ell}\|W_{\ell}\|_{\sigma} $. Nevertheless, there are many examples of (non-positive) weight matrices such that $ \|\prod_{\ell}|W_{\ell}|\|_{\sigma} \leq \prod_{\ell}\|W_{\ell}\|_{\sigma} $.

Several results have identified product of weight norms as a complexity measure; we detail a few here. For example, \cite{Bartlett2017} use covering numbers to obtain the generalization bound
\begin{align}
    \widetilde{\mathcal{O}}\Big(\frac{\max_{x\in S}\|x\|_2L^{3/2}\overbar{R}^{3/2}\prod_1^L \|W_\ell\|_\sigma }{\gamma \sqrt{n}}\Big),
\end{align}
where $\overbar{R} = \frac{1}{L}\sum_{1}^L \big(\frac{\|W_\ell\|_{2,1}}{\|W_\ell\|_{\sigma}}\big)^{2/3}$, which, similar to our $\zeta$ quantities, admits an interpretation as an average \emph{effective} width. This is naturally contrasted with our bound in terms of $\|\prod_1^L |W_\ell|\|_\sigma$ and $\zeta_2$, though as we mention above, the quantities $\|\prod_1^L |W_\ell|\|_\sigma$ and $\prod_1^L \|W_\ell\|_\sigma$ are not directly comparable. However, there are examples of matrices for which our bound is superior. Similar bounds were likewise obtained via a PAC-Bayes approach in \cite{Neyshabur2018ANetworks}, though these are known to be strictly weaker than the above bound from \cite{Bartlett2017}. 

The approach of \cite{NoahGolowichAlexanderRakhlin2018} (which addressed the single output case) used a more direct bound on Rademacher complexities, via `peeling', to get generalization bounds of order 
\begin{align}
    \frac{\max_{x\in S}\|x\|_2\sqrt{L+\log(d)}\prod_1^L \|W_\ell\|}{\gamma \sqrt{n}},
\end{align}

where the associated norm is $\|\cdot\|_{1,\infty}$ or $\|\cdot\|_F$. Notably, these bounds avoid the correction factors $\overbar{R}$ or $\zeta$ appearing in our results and \cite{Bartlett2017,Neyshabur2018ANetworks}, and have slightly more mild (explicit) dependence on the number of layers. However, since these bounds are in terms of entry-wise norms, our capacity constants can be shown to lower bound these quantities. For example, in the single output case, $\mathscr{V}_2$ will lower bound the product $\prod_1^L \|W_\ell\|_F$. Similarly, taking $q=1$ in Corollary \ref{thm:norm-based-bounds}, we get a bound in terms of $\|\prod_1^L |W_\ell|\|_{1,\infty}$, which lower bounds $\prod_1^L \|W_\ell\|_{1,\infty}$.

Other norm-based bounds have identified more global quantities as complexity measures for deep networks. Of particular relevance to our work is the path norm $\phi_{p}$ of a network $f(x;W)$ studied in \cite{Neyshabur2015Norm-BasedNetworks,Neyshabur2015Path-SGD:Networks}, which is defined as\footnote{Note this is in fact a generalization of the definition in \cite{Neyshabur2015Norm-BasedNetworks}, which considered only the single output case, and hence did not have the sum over $j_L$.} 
\begin{align}
    \phi_p = \sum_{j_L}\Big(\sum_{j_0,...,j_{L-1}}|w_{j_0,...,j_{L}}|^p\Big)^{1/p}.
\end{align}
We observe immediately that $\phi_1 = \mathscr{V}_1$. In \cite{Neyshabur2015Norm-BasedNetworks}, generalization bounds of order
\begin{align}
    \frac{\max_{x\in S}\|x\|_\infty \mathscr{V}_1 2^L}{\gamma \sqrt{n}},
\end{align} 
were given for the case of $p=1$ and $k = 1$.
While this bound avoids involving products of norms, it has explicit exponential dependence on the depth of order $2^L$, which our bound improves to the low order polynomial $L^{3/2}$. Furthermore, the following lemma shows that $\mathscr{V}_2$ may be upper bounded by $\phi_2$. Hence we can view our results also as an improvement on this line of work.
\begin{lemma}
We have
\begin{align}
\mathscr{V}_2 \leq \sum_{j_L}\big(\sum_{j_0,j_1,...,j_{L-1}} w_{j_0,j_1,...,j_L}^2\big)^{1/2}
\end{align}
where in the single output case, the right-hand side is equal to the $2-$path norm $\phi_2$ from \cite{Neyshabur2015Norm-BasedNetworks}.
\end{lemma}
\begin{proof}
See \ref{bounding-normalizing-constants}.
\end{proof}

Finally, we remark that while more direct analysis of Rademacher complexities, such as \cite{NoahGolowichAlexanderRakhlin2018, Neyshabur2015Norm-BasedNetworks}, avoid the correction factors such as $\zeta$ and $\overbar{R}$, these works seem to address only the single output case. It is therefore unclear if extending these analyses to the multi-class setting would involve more direct dependence on the number of classes $k$.

\section{Empirical Investigation}
\label{empirical-investigation}
In the previous section, we used a sampling procedure as a technical tool to derive generalization bounds. Intuitively, the ability to express a large network with many parameters as a network with few parameters of low precision indicates lower complexity. In this section, we investigate this relationship empirically, using the sampling procedure employed theoretically above. For simplicity, we work with $\mathscr{V}_1$ and the 1-path distribution. For each network $\mathscr{V}_1f(x;p)$ considered, we will draw Multinomial$(M,p)$ samples and compute the corresponding estimates $\tilde{p}_M$. Here, as is justified in Section 2, we will use the number of trials $M$ as a proxy for compression, and investigate the number of samples $M$ required to obtain a given level of accuracy. We note that working directly with the full path distribution $p$ quickly becomes unwieldy as the network grows in depth, as it involves storing a (potentially dense) $L$-tensor. Fortunately, by exploiting the Markov structure of $p$ and storing only the conditional distributions $p_{j_\ell|j_{\ell+1}}$, and sampling forward through the Markov chain, this issue can be avoided.

We study three different problems, which range from easy (generalize very well) to hard (generalize poorly). Namely, we study a basic \mnist\, network, a \cifar\, network, and a \cifar\, network with labels chosen uniformly at random. We use a four layer, feed-forward network with hidden layer dimensions $600,400$ and $200$. Throughout our experiments, plots demonstrating the performance of the \mnist\, (easy; test accuracy $\approx 98\%$) network will be shown in orange \tikz\draw[orange,fill=orange] (0,0) circle (.5ex);, the \cifar\, (medium; test accuracy $\approx 60\%$) network in green \tikz\draw[color={rgb:red,18;green,54.5;blue,34.1},fill={rgb:red,18;green,54.5;blue,34.1}] (0,0) circle (.5ex);, and the \cifar\, with random labels (hard; test accuracy $\approx 10\%$) network in blue \tikz\draw[color={rgb:red,00;green,23;blue,66},fill={rgb:red,00;green,23;blue,66}] (0,0) circle (.5ex);. Each network is trained to 100\% training accuracy using stochastic gradient descent with momentum set to 0.9 and no additional regularization. 

 \begin{figure}
\hfill
\subfigure{\includegraphics[width=7.1cm]{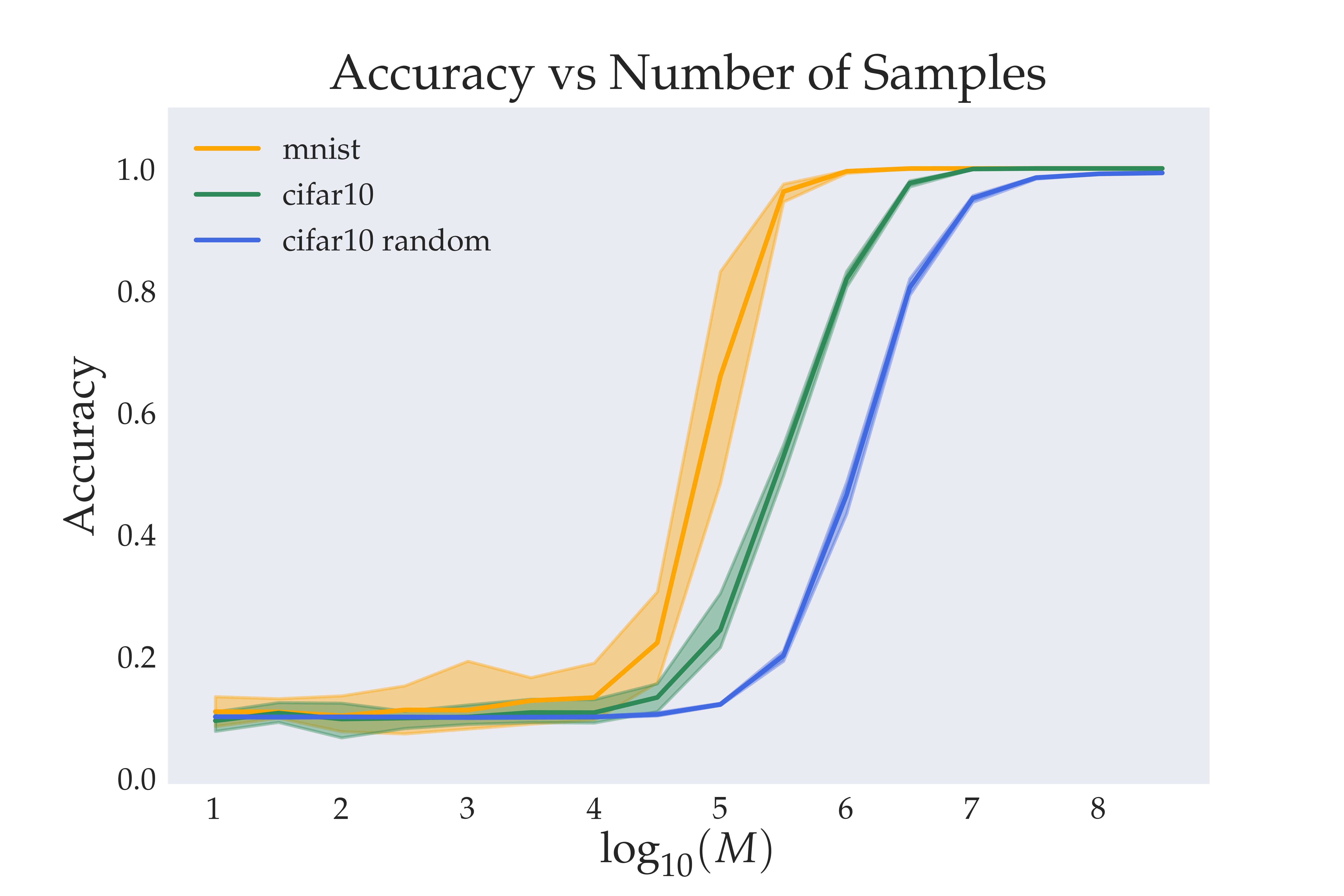}}
\hfill
\subfigure{\includegraphics[width=7.1cm]{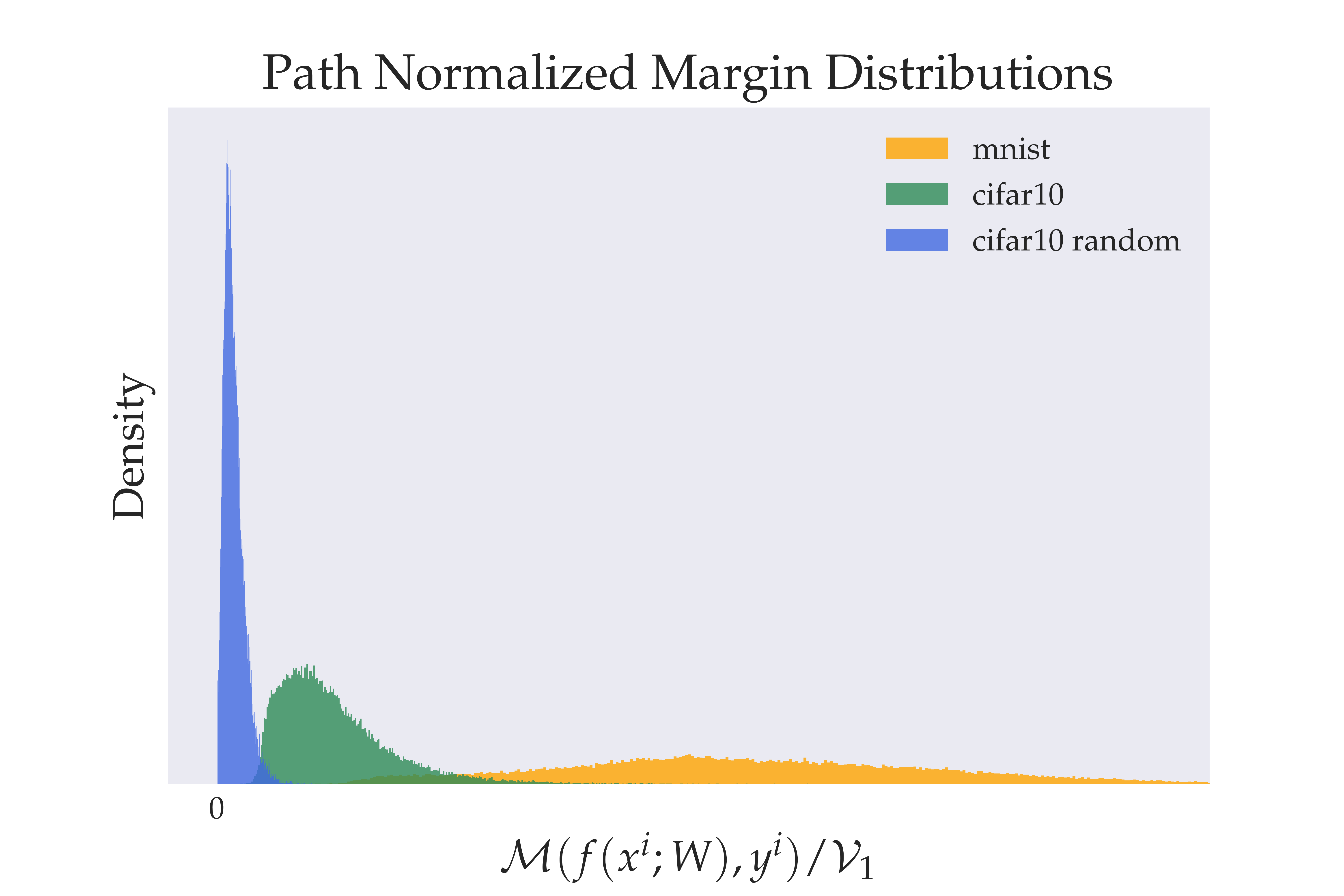}}
\hfill
\caption{\textbf{Left}: Comparison of sampling with \mnist, \cifar, and \cifar\, with random labels. Solid line indicates mean classification accuracy on the training set over 10 trials; shaded region indicates range over these trials. \textbf{Right}: Comparison of path normalized margin distributions. We observe that the \mnist\ model has considerably larger normalized margins than the \cifar\ model, which itself has larger normalized margins than \cifar\, with random labels.}
\label{fig:acc_comp_binary}
\end{figure}

We observe that compressibility does indeed correlate with generalization: networks with higher test accuracy can be represented with fewer samples. For example, we see from Figure \ref{fig:acc_comp_binary} that with $M=10^6$ samples, we can obtain $100\%$ accuracy from the \mnist\ model, $80\%$ from the \cifar\ model, and only $40\%$ from the \cifar\ model with random labels. 

According to our theory, this accuracy should be governed in part by the trade-off of $\mathscr{V}_1/\gamma$ with $\hat{\ell}_\gamma(f)$, where $\gamma$ is some chosen value of margin. To study this, we look at the \emph{path normalized margin distribution}, which for a network $f(x;W)$ and dataset $\{(x^i,y^i)\}_{i=1}^n$, is the histogram of values
\begin{align}
  \frac{\mathcal{M}(f(x^i,W),y^i)}{\mathscr{V}_1}.  
\end{align}
This is analogous to the normalized margins studied in \cite{Bartlett2017}, though with the path variation serving as the normalizing constant, rather than the spectral complexity (which is related to the network's Lipschitz constant). In Figure \ref{fig:acc_comp_binary}, we see that these distributions do indeed significantly distinguish the three models. Intuitively, it seems natural that larger margin classifiers would be easier to sparsely approximate, as correspondingly larger perturbations to the function's output do not change the function's classification decision. It is also well documented, and is suggested by Theorem \ref{thm:gen-bound-mt}, that large (normalized) margins play an important role in generalization behavior of neural network classifiers. Thus we see that, in this sense, model sparsification and compression are strongly related to generalization.

\section{Conclusions and Future Directions}
In this paper, we exploited the Markov structure of positive homogeneous networks to analyze and implement a sampling scheme for network sparsification, which we then used to obtain covering number and generalization bounds. Our analysis identified the path variations $\mathscr{V}_q$, which we show to be bounded by various norms $\|\prod_1^L |W_\ell|\|$, as important quantities controlling approximation rates and generalization error, which we then verified empirically. In what follows, we briefly highlight some potential directions for further work building on the present results.

%\textbf{Sampling with residual networks.} Because of the presence of skip connections, our results do not directly apply to the analysis of residual networks. However, using an `unravelled' view recently proposed in \cite{Veit2016ResidualNetworks}, it is possible to view residual networks as a collection of paths with different lengths. Using this formulation, we show in the Supplementary Material how to define a sampling procedure in terms of a \textit{mixture of path distributions}. It remains a fascinating direction for future work to analyze this sampling procedure, and perhaps obtain bounds that lend insight into the impressive empirical success of residual networks.

\textbf{Sampling with residual networks.} Residual networks have been shown to be a powerful network architecture, used to obtain state-of-the-art performance on many difficult classification tasks. However, our analysis does not immediately apply to networks with skip connections. Here we present one potential direction for extending our techniques to the analysis of residual networks.  

Recently, \cite{Veit2016ResidualNetworks} proposed an `unravelled' view of residual networks as a collection of paths with different lengths. We show how to utilize this perspective to develop a sampling strategy.
%That is,
%$$
%p = \sum_{L'=0}^L (1/2)^{L'}\binom{L}{L'}p_{L'},
%$$
%where $p_{L'}$ is the path distribution for (fully-connected) sub-networks of the ResNet consisting of paths of length $ L' $.
Each unique path $P=(j_{\ell_0},j_{\ell_1},\dots, j_{\ell_m})$ through the residual network can be assigned a binary code $b(P) \in \{0, 1\}^L $ where $b_t(P) = 1$ if the input flows through residual module $ t $ and $0$ if it is skipped (i.e., a skip connection). In this case, the path distribution of a residual network can be seen as a mixture of path distributions for fully-connected networks, namely,
$
p = 2^{-L}\sum_{b\in\{0, 1\}^L} p^{(b)},
$
where $ p^{(b)} $ is the path distribution of the fully-connected subnetwork induced by all paths $P$ such that $ b=b(P) $. 

Thus, the marginal distribution of paths leading up to residual module $t$ is a mixture of $2^{t-1}$ different path distributions generated from every possible configuration of the previous $t-1$ residual modules. Note also that $ p $ generates paths with lengths that are distributed $ \text{Binom}(1/2, L) $. This coincides with the model of path lengths proposed in \cite{Veit2016ResidualNetworks}, who empirically show that they are distributed $ \text{Bin}(1/2,L)$ and concentrate around $ L/2 $. Samples from $ p $ can easily be generated by first sampling $ \tilde b $ from the uniform mixing distribution on $ \{0, 1\}^L $ and then sampling a path from the Markov distribution $ p^{(\tilde b)} $. Counts $ K $ of indices can be used to form the empirical Markov distribution $ \tilde p $ as before.

\textbf{Removing the path complexity.} A notable difference between our results and those of \cite{NoahGolowichAlexanderRakhlin2018} is the presence of the path complexity $\zeta$. As a similar term also appears in \cite{Bartlett2017}, it seems as though such correction factors may be consequences of the covering approach. By using a `peeling' argument, and bounding the Rademacher complexity directly, \cite{NoahGolowichAlexanderRakhlin2018} are able to avoid such factors. It is an interesting open question whether the path complexity term can be removed from our bounds using similar techniques.

\nocite{*}
\bibliographystyle{unsrt}
\bibliography{references}

\begin{thebibliography}{31}
\providecommand{\natexlab}[1]{#1}
\providecommand{\url}[1]{\texttt{#1}}
\expandafter\ifx\csname urlstyle\endcsname\relax
  \providecommand{\doi}[1]{doi: #1}\else
  \providecommand{\doi}{doi: \begingroup \urlstyle{rm}\Url}\fi

\bibitem[Achlioptas et~al.(2013)Achlioptas, Karnin, and
  Liberty]{Achlioptas2013MatrixBest}
Dimitris Achlioptas, Zohar Karnin, and Edo Liberty.
\newblock {Matrix Entry-wise Sampling: Simple is Best}.
\newblock In \emph{KDD}, 2013.
\newblock URL
  \url{https://pdfs.semanticscholar.org/aa64/b8fb3382e42f90ee93a1dd0c78f13833963f.pdf}.

\bibitem[Arora et~al.(2018)Arora, Ge, Neyshabur, and Zhang]{Arora2018}
Sanjeev Arora, Rong Ge, Behnam Neyshabur, and Yi~Zhang.
\newblock {Stronger Generalization Bounds for Deep Nets via a Compression
  Approach}.
\newblock In \emph{35th International Conference on Machine Learning}, 2018.
\newblock URL \url{http://proceedings.mlr.press/v80/arora18b/arora18b.pdf}.

\bibitem[Bailey et~al.(2019)Bailey, Ji, Telgarsky, and
  Xian]{Bailey2019ApproximationNetworks}
Bolton Bailey, Ziwei Ji, Matus Telgarsky, and Ruicheng Xian.
\newblock {Approximation power of random neural networks}.
\newblock 2019.
\newblock URL \url{https://arxiv.org/pdf/1906.07709.pdf}.

\bibitem[Barron(1991)]{Barron1991ComplexityNetworks}
Andrew~R. Barron.
\newblock {Complexity Regularization with Application to Artificial Neural
  Networks}.
\newblock In \emph{Nonparametric Functional Estimation and Related Topics},
  pages 561--576. Springer Netherlands, Dordrecht, 1991.
\newblock \doi{10.1007/978-94-011-3222-0{\_}42}.
\newblock URL \url{http://link.springer.com/10.1007/978-94-011-3222-0_42}.

\bibitem[Barron(1993)]{Barron1993UniversalFunction}
Andrew~R. Barron.
\newblock {Universal approximation bounds for superpositions of a sigmoidal
  function}.
\newblock \emph{IEEE Transactions on Information Theory}, 39\penalty0
  (3):\penalty0 930--945, 5 1993.
\newblock ISSN 0018-9448.
\newblock \doi{10.1109/18.256500}.
\newblock URL \url{https://ieeexplore.ieee.org/document/256500/}.

\bibitem[Barron and Klusowski(2018)]{Barron2018ApproximationNetworks}
Andrew~R. Barron and Jason~M. Klusowski.
\newblock {Approximation and Estimation for High-Dimensional Deep Learning
  Networks}.
\newblock 2018.
\newblock URL \url{https://arxiv.org/pdf/1809.03090.pdf}.

\bibitem[Bartlett(1998)]{Bartlett1998TheNetwork}
Peter~L. Bartlett.
\newblock {The sample complexity of pattern classification with neural
  networks: the size of the weights is more important than the size of the
  network}.
\newblock \emph{IEEE Transactions on Information Theory}, 44\penalty0
  (2):\penalty0 525--536, 3 1998.
\newblock ISSN 00189448.
\newblock \doi{10.1109/18.661502}.
\newblock URL \url{http://ieeexplore.ieee.org/document/661502/}.

\bibitem[Bartlett et~al.(2017)Bartlett, Foster, and Telgarsky]{Bartlett2017}
Peter~L. Bartlett, Dylan~J. Foster, and Matus Telgarsky.
\newblock {Spectrally-normalized margin bounds for neural networks}.
\newblock In \emph{31st Conference on Neural Information Processing Systems},
  2017.
\newblock URL
  \url{https://papers.nips.cc/paper/7204-spectrally-normalized-margin-bounds-for-neural-networks.pdf}.

\bibitem[Bartlett et~al.(2019)Bartlett, Harvey, Liaw, and
  Mehrabian]{Bartlett2019Nearly-tightNetworks}
Peter~L. Bartlett, Nick Harvey, Christopher Liaw, and Abbas Mehrabian.
\newblock {Nearly-tight VC-dimension and Pseudodimension Bounds for Piecewise
  Linear Neural Networks}.
\newblock \emph{Journal of Machine Learning Research}, 20:\penalty0 1--17,
  2019.
\newblock URL \url{http://jmlr.org/papers/v20/17-612.html.}

\bibitem[Baykal et~al.(2019)Baykal, Liebenwein, Gilitschenski, Feldman, and
  Rus]{Baykal2019}
Cenk Baykal, Lucas Liebenwein, Igor Gilitschenski, Dan Feldman, and Daniela
  Rus.
\newblock {Data-Dependent Coresets for Compressing Neural Networks with
  Applications to Generalization Bounds}.
\newblock In \emph{International Conference on Learning Representations}, 2019.

\bibitem[Cheng et~al.(2018)Cheng, Wang, Zhou, and
  Zhang]{Cheng2018ModelChallenges}
Yu~Cheng, Duo Wang, Pan Zhou, and Tao Zhang.
\newblock {Model Compression and Acceleration for Deep Neural Networks: The
  Principles, Progress, and Challenges}.
\newblock \emph{IEEE Signal Processing Magazine}, 35\penalty0 (1):\penalty0
  126--136, 1 2018.
\newblock ISSN 1053-5888.
\newblock \doi{10.1109/MSP.2017.2765695}.
\newblock URL \url{http://ieeexplore.ieee.org/document/8253600/}.

\bibitem[Choromanska et~al.(2015)Choromanska, Henaff, Mathieu, Arous, and
  Lecun]{Choromanska2015TheNetworks}
Anna Choromanska, Mikael Henaff, Michael Mathieu, Gérard~Ben Arous, and Yann
  Lecun.
\newblock {The Loss Surfaces of Multilayer Networks}.
\newblock In \emph{18th International Conference on Artificial Intelligence and
  Statistics}, 2015.
\newblock URL \url{https://arxiv.org/pdf/1412.0233.pdf}.

\bibitem[de~Azevedo~Pribitkin(2009)]{deAzevedoPribitkin2009SimpleFunctions}
Wladimir de~Azevedo~Pribitkin.
\newblock {Simple upper bounds for partition functions}.
\newblock \emph{The Ramanujan Journal}, 18\penalty0 (1):\penalty0 113--119, 1
  2009.
\newblock ISSN 1382-4090.
\newblock \doi{10.1007/s11139-007-9022-z}.
\newblock URL \url{http://link.springer.com/10.1007/s11139-007-9022-z}.

\bibitem[Drineas and Zouzias(2011)]{Drineas2011AInequality}
Petros Drineas and Anastasios Zouzias.
\newblock {A note on element-wise matrix sparsification via a matrix-valued
  Bernstein inequality}.
\newblock \emph{Information Processing Letters}, 111\penalty0 (8):\penalty0
  385--389, 3 2011.
\newblock ISSN 0020-0190.
\newblock \doi{10.1016/J.IPL.2011.01.010}.
\newblock URL
  \url{https://www.sciencedirect.com/science/article/pii/S0020019011000196}.

\bibitem[Dziugaite and Roy(2017)]{Dziugaite2017ComputingData}
Gintare~Karolina Dziugaite and Daniel~M Roy.
\newblock {Computing Nonvacuous Generalization Bounds for Deep (Stochastic)
  Neural Networks with Many More Parameters than Training Data}.
\newblock In \emph{Uncertainty in Artificial Intelligence (UAI)}, 2017.
\newblock URL \url{http://auai.org/uai2017/proceedings/papers/173.pdf}.

\bibitem[Frankle and Carbin(2019)]{Frankle2019TheNetwork}
Jonathan Frankle and Michael Carbin.
\newblock {The Lottery Ticket Hypothesis: Finding Sparse, Trainable Neural
  Network}.
\newblock In \emph{International Conference on Learning Representations}, 2019.
\newblock URL \url{https://openreview.net/pdf?id=rJl-b3RcF7}.

\bibitem[Golowich et~al.(2018)Golowich, Rakhlin, and
  Shamir]{NoahGolowichAlexanderRakhlin2018}
Noah Golowich, Alexander Rakhlin, and Ohad Shamir.
\newblock {Size-Independent Sample Complexity of Neural Networks}.
\newblock In \emph{31st Conference On Learning Theory}, 2018.

\bibitem[Kawaguchi et~al.(2019)Kawaguchi, Kaelbling, and
  Bengio]{Kawaguchi2019GeneralizationLearning}
Kenji Kawaguchi, Leslie~Pack Kaelbling, and Yoshua Bengio.
\newblock {Generalization in Deep Learning}.
\newblock Technical report, MIT, 2019.
\newblock URL \url{https://arxiv.org/pdf/1710.05468.pdf}.

\bibitem[Kundu and Drineas(2014)]{Kundu2014ASparsification}
Abhisek Kundu and Petros Drineas.
\newblock {A Note on Randomized Element-wise Matrix Sparsification}.
\newblock Technical report, 2014.
\newblock URL \url{https://arxiv.org/pdf/1404.0320.pdf}.

\bibitem[Liang et~al.(2019)Liang, Poggio, Rakhlin, and
  Stokes]{Liang2019Fisher-RaoNetworks}
Tengyuan Liang, Tomaso Poggio, Alexander Rakhlin, and James Stokes.
\newblock {Fisher-Rao Metric, Geometry, and Complexity of Neural Networks}.
\newblock In \emph{22nd International Conference on Artificial Intelligence and
  Statistics}, 2019.
\newblock URL \url{http://proceedings.mlr.press/v89/liang19a/liang19a.pdf}.

\bibitem[Mohri et~al.(2018)Mohri, Rostamizadeh, and Talwalkar]{Mohri2018}
Mehryar Mohri, Afshin Rostamizadeh, and Ameet Talwalkar.
\newblock \emph{{Foundations of Machine Learning}}.
\newblock MIT Press, second edition, 2018.

\bibitem[Nagarajan and Kolter(2019)]{Nagarajan2019}
Vaishnavh Nagarajan and J~Zico Kolter.
\newblock {Deterministic PAC-Bayesion Generalization Bounds for Deep Network
  Via Generalizing Noise-Resilience}.
\newblock In \emph{International Conference on Learning Representations}, 2019.
\newblock URL \url{https://openreview.net/pdf?id=Hygn2o0qKX}.

\bibitem[Neyshabur et~al.(2015{\natexlab{a}})Neyshabur, Salakhutdinov, and
  Srebro]{Neyshabur2015Path-SGD:Networks}
Behnam Neyshabur, Ruslan~R. Salakhutdinov, and Nati Srebro.
\newblock {Path-SGD: Path-Normalized Optimization in Deep Neural Networks}.
\newblock In \emph{28th Conference on Neural Information Processing Systems},
  2015{\natexlab{a}}.
\newblock URL
  \url{https://papers.nips.cc/paper/5797-path-sgd-path-normalized-optimization-in-deep-neural-networks}.

\bibitem[Neyshabur et~al.(2015{\natexlab{b}})Neyshabur, Tomioka, and
  Srebro]{Neyshabur2015Norm-BasedNetworks}
Behnam Neyshabur, Ryota Tomioka, and Nathan Srebro.
\newblock {Norm-Based Capacity Control in Neural Networks}.
\newblock In \emph{28th Conference on Learning Theory}, 2015{\natexlab{b}}.
\newblock URL \url{http://proceedings.mlr.press/v40/Neyshabur15.html}.

\bibitem[Neyshabur et~al.(2018)Neyshabur, Bhojanapalli, and
  Srebro]{Neyshabur2018ANetworks}
Behnam Neyshabur, Srinadh Bhojanapalli, and Nathan Srebro.
\newblock {A PAC-Bayesian Approach to Spectrally-Normalized Margin Bound for
  Neural Networks}.
\newblock In \emph{International Conference on Learning Representations}, 2018.
\newblock URL \url{https://openreview.net/pdf?id=Skz_WfbCZ}.

\bibitem[Pisier and Maurey(1980)]{Pisier1980RemarquesMaurey}
G.~Pisier and B.~Maurey.
\newblock {Remarques sur un r{\'{e}}sultat non publi{\'{e}} de B. Maurey}.
\newblock \emph{S{\'{e}}minaire Analyse fonctionnelle (dit "Maurey-Schwartz")},
  pages 1--12, 1980.
\newblock URL \url{http://www.numdam.org/item/SAF_1980-1981____A5_0/}.

\bibitem[Sridharan()]{SridharanNoteBound}
Karthik Sridharan.
\newblock {Note on Refined Dudley Integral Covering Number Bound}.
\newblock URL \url{https://www.cs.cornell.edu/~sridharan/dudley.pdf}.

\bibitem[Veit et~al.(2016)Veit, Wilber, and Belongie]{Veit2016ResidualNetworks}
Andreas Veit, Michael~J. Wilber, and Serge Belongie.
\newblock {Residual Networks Behave Like Ensembles of Relatively Shallow
  Networks}.
\newblock In \emph{29th Conference on Neural Information Processing Systems},
  2016.
\newblock URL
  \url{https://papers.nips.cc/paper/6556-residual-networks-behave-like-ensembles-of-relatively-shallow-networks}.

\bibitem[Yarotsky(2017)]{Yarotsky2017ErrorNetworks}
Dmitry Yarotsky.
\newblock {Error bounds for approximations with deep ReLU networks}.
\newblock \emph{Neural Networks}, 94:\penalty0 103--114, 10 2017.
\newblock ISSN 0893-6080.
\newblock \doi{10.1016/J.NEUNET.2017.07.002}.
\newblock URL
  \url{https://www.sciencedirect.com/science/article/pii/S0893608017301545}.

\bibitem[Zhang(2002)]{Zhang2002CoveringClasses}
Tong Zhang.
\newblock {Covering Number Bounds of Certain Regularized Linear Function
  Classes}.
\newblock \emph{Journal of Machine Learning Research}, 2:\penalty0 527--550,
  2002.
\newblock URL \url{http://www.jmlr.org/papers/volume2/zhang02b/zhang02b.pdf}.

\bibitem[Zhou et~al.(2019)Zhou, Veitch, Austern, Adams, and
  Orbanz]{Zhou2019Non-VacuousApproach}
Wenda Zhou, Victor Veitch, Morgane Austern, Ryan~P Adams, and Peter Orbanz.
\newblock {Non-Vacuous Generalization Bounds at the ImageNet Scale: A
  PAC-Bayesian Compression Approach}.
\newblock In \emph{International Conference on Learning Representations}, 2019.
\newblock URL \url{https://openreview.net/pdf?id=BJgqqsAct7}.

\end{thebibliography}

\appendix

\section{Proofs of main results}
\begin{theorem}
Let $f(x;W)$ be an $L$-layer ReLU network, $S$ a dataset, and let $1 \leq q\leq 2$. If $\tilde{p}$ is the Markov distribution formed from $M$ samples from $p^{(q)}_{j_0,j_1,...,j_L}$, then 
    \begin{align}
        \E_{\tilde{p}}\Big[\frac{1}{n}\sum_{x\in S}\|f(x;\widetilde{W}) - f(x;W)\|_2^2 \Big]
    \leq \Big(\frac{\mathscr{V}_{q}\zeta_q L}{\sqrt{M}}\Big)^2,
    \end{align}
where $f(x;\widetilde{W}) = \mathscr{V}_qf(x;\tilde{p})$. 
\label{thm:maximal-bound-appendix}
\end{theorem}

\begin{proof}
The proofs are the same for each $p^{(q)}$, so we write $p$ for a generic path distribution, and $\mathscr{V}$ for a generic path variation (with the understanding that in the spectral case, we are considering $\ell_2$ control on the inputs).

We can decompose the difference $f(x;p) - f(x;\tilde{p})$ into a telescoping sum
\begin{equation} \label{eq:telescope}
    f(x;p) - f(x;\tilde{p}) = \sum_{\ell=1}^{L}[f^{\ell+1}(x;p,\tilde{p}) - f^\ell(x;p,\tilde{p})]
\end{equation}
in which $ f^{\ell+1}(x;p,\tilde{p}) $ and $ f^\ell(x;p,\tilde{p}) $ differ only on layer $ \ell $, the former using $ p_{j_{\ell-1}|j_{\ell}} $ and the later using $ \tilde p_{j_{\ell-1}|j_{\ell}} $.
That is, for each output unit $j_L$ we let
\begin{align*}
    f^\ell(x; p,\tilde{p})_{j_L} &= \sum_{j_{L-1}}\tilde{p}_{j_L}\tilde{p}_{j_{L-1}|j_L}\phi\Big(\sum_{j_{L-2}} \tilde{p}_{j_{L-2}|j_{L-1}}\phi\Big(\\
    & \qquad\cdots \phi\Big(\sum_{j_{\ell-1}}\tilde{p}_{j_{\ell-1}| j_{\ell}}\phi\Big(\sum_{j_{\ell-2}}p_{{j_{\ell-2 }}|j_{\ell-1}}\Big(\cdots \phi\Big(\sum_{j_0} p_{j_0|j_1}x_{j_0}\Big)\Big)\Big)\Big)\Big)\Big).
\end{align*}
In other words, $ f^\ell(x; p,\tilde{p}) $ is a network with weight matrices $ (P_1, \dots, P_{\ell-1}, \widetilde P_{\ell}, \widetilde P_{\ell+1}, \dots, \widetilde P_L) $, where $ P_{\ell}[j_{\ell},j_{\ell-1}] = p_{j_{\ell-1}|j_{\ell}}$ and $ \widetilde P_{\ell}[j_{\ell},j_{\ell-1}] = \tilde p_{j_{\ell-1}|j_{\ell}} $ are transition matrices for the Markov distributions $ p $ and $\tilde p$, respectively. Now let $\P_S$ be the empirical distribution for the sample $S$. Then using the triangle inequality for the $L_2$ norm associated with the joint distribution of $\tilde{p}$ and $\P_S$, we observe
\begin{align}
    \E_{\tilde{p}}\Big[\frac{1}{n}\sum_{x\in S}\|f(x;p)-f(x;\tilde{p})\|_2^2 \Big] = \E_{\tilde{p},\P_S}[\|f(x;p)-f(x;\tilde{p})\|_2^2 ] \leq \Big(\sum_{\ell} E_\ell\Big)^2
    \label{eqn:sum-bound}
\end{align}

where $$ E_\ell = \Big(\E_{\tilde{p}}\Big[\frac{1}{n}\sum_{x\in S}\|f^{\ell+1}(x;p,\tilde{p}) - f^\ell(x;p,\tilde{p})\|_2^2\Big]\Big)^{1/2}.$$
We are therefore interested in bounding $$
\E_{\tilde{p}}\Big[\frac{1}{n}\sum_{x\in S}\sum_{j_L}|f^{\ell+1}(x;p,\tilde{p})_{j_L} - f^\ell(x;p,\tilde{p})_{j_L}|^2\Big]
$$ 
for each $\ell$. For $\ell = L$ we have
\begin{align*}
    \frac{1}{n}\sum_{x\in S}\sum_{j_L}|f^{L+1}(x;p,\tilde{p})_{j_L} - f^{L}(x;p,\tilde{p})_{j_L}|^2 = \frac{1}{n}\sum_{x\in S}\sum_{j_L}\Big|\sum_{j_{L-1}}(\tilde{p}_{j_L,j_{L-1}} - p_{j_L,j_{L-1}})x_{j_{L-1}}(x)\Big|^2,
\end{align*}
where $x_{j_{L-1}}(x)$ is the output of the network at the $j_{L-1}$th node entering the last layer. Then noticing that $\tilde{p}_{j_L,j_{L-1}} = \frac{1}{M}\sum_{i=1}^M \mathbb{1}((\tj,\tj') = (j_L,j_{L-1}))$, where $\mathbb{1}((\tj,\tj') = (j_L,j_{L-1})) \sim$ Bern$(p_{j_L,j_{L-1}})$, we may calculate
\begin{align*}
    &\frac{1}{n}\sum_{x\in S}\sum_{j_L}\E\left[\Big|\sum_{j_{L-1}}(\tilde{p}_{j_L,j_{L-1}} - p_{j_L,j_{L-1}})x_{j_{L-1}}(x)\Big|^2\right]\\
    &= \frac{1}{n}\sum_{x\in S}\sum_{j_L}\Big(\Big[\sum_{j_{L-1}} p_{j_L,j_{L-1}}(x_{j_{L-1}}(x') - z_{j_L})^2\Big] + (1-p_{j_L})z_{j_L}^2\Big) \\
    &\leq \frac{1}{n}\sum_{x\in S}\frac{1}{M}\sum_{j_L}\sum_{j_{L-1}} p_{j_L,j_{L-1}} x^2_{j_{L-1}}(x')
\end{align*}
where the last inequality follows from the fact that the MSE is minimized at the mean (so we can upper bound this term by plugging in $z_{j_L}=0$). Using the Lipschitz property of $\phi$, we have 
\begin{align*}
\frac{1}{n}\sum_{x\in S}x^2_{j_{L-1}}(x') & \leq \frac{1}{n}\sum_{x\in S}(\sum_{j_{L-2},j_{L-3},\dots,j_0}p_{j_{L-2},\dots,j_0|j_{L-1}}|x'_{j_0}|)^2 \\
& = \frac{1}{n}\sum_{x\in S}(\sum_{j_0}p_{j_0|j_{L-1}}|x'_{j_0}|)^2 \\
& \leq \frac{1}{n}\sum_{x\in S}(\sum_{j_0}p_{j_0|j_{L-1}}|x'_{j_0}|^{q^*})^{2/q^*}\\
&\leq \Big(\frac{1}{n}\sum_{x\in S}\sum_{j_0}p_{j_0|j_{L-1}}|x'_{j_0}|^{q^*}\Big)^{2/q^*}
\end{align*}
where the last two inequalities follow from Jensen's inequality (since for $1\leq q\leq 2$, we have $q^* \geq 2$, and hence $z^{q^*}$ is convex and $z^{2/q^*}$ is concave). 

Next, we observe
$$
\frac{1}{n}\sum_{x\in S}\sum_{j_0}p_{j_0|j_{L-1}}|x'_{j_0}|^{q^*} \leq \frac{1}{n}\max_{j_0}\sum_{x\in S} |x'_{j_0}|^{q^*} = 1.
$$
Hence we have
\begin{align*}
    \frac{1}{n}\sum_{x\in S}\frac{1}{M}\sum_{j_L}\sum_{j_{L-1}} p_{j_L,j_{L-1}} x^2_{j_{L-1}}(x')
    &\leq \frac{1}{n}\sum_{x\in S}\frac{1}{M}\sum_{j_L}\sum_{j_{L-1}} p_{j_L,j_{L-1}} = \frac{1}{M}.
\end{align*}

For $ \ell = 1, 2, \dots, L-1 $, repeated application of the Lipschitz property of $ \phi $ permits bounding each difference $ \sum_{j_L} |f^{\ell+1}_{j_L}(x; p, \tilde p)- f^{\ell}_{j_L}(x; p, \tilde p)|^2 $ by
\begin{align*} \label{eq:diff}
\sum_{j_L}(\sum_{j_{L-1}, \dots, j_{\ell+1}}\tilde p_{j_L,\dots,j_{\ell+1}}|\sum_{j_{\ell}} \tilde p_{j_{\ell}| j_{\ell+1}}(\phi(\tilde z_{j_{\ell}})-\phi(z_{j_{\ell}}))|)^2 \\
= \sum_{j_L}(\sum_{j_{\ell+1}}\tilde p_{j_L,j_{\ell+1}}|\sum_{j_{\ell}} \tilde p_{j_{\ell}| j_{\ell+1}}(\phi(\tilde z_{j_{\ell}})-\phi(z_{j_{\ell}}))|)^2
\end{align*}
where $z_{j_{\ell}} = \sum_{j_{\ell-1}} p_{j_{\ell-1}|j_{\ell}} x_{j_{\ell-1}} $ and $ \tilde z_{j_{\ell}} = \sum_{j_{\ell-1}} \tilde p_{j_{\ell-1}|j_{\ell}} x_{j_{\ell-1}} $.
Since the quantities on the inside of the square are non-negative, and the sum of squares is less than the square of the sum, we have that this is at most
$$
(\sum_{j_L}\sum_{j_{\ell+1}}\tilde p_{j_L,j_{\ell+1}}|\sum_{j_{\ell}} \tilde p_{j_{\ell}| j_{\ell+1}}(\phi(\tilde z_{j_{\ell}})-\phi(z_{j_{\ell}}))|)^2
= (\sum_{j_{\ell+1}}\tilde p_{j_{\ell+1}}|\sum_{j_{\ell}} \tilde p_{j_{\ell}| j_{\ell+1}}(\phi(\tilde z_{j_{\ell}})-\phi(z_{j_{\ell}}))|)^2
$$
Using the triangle inequality and marginalizing, we get the further upper bound of 
$$
(\sum_{j_\ell}\tilde{p}_{j_\ell}|\phi(\tilde z_{j_{\ell}})-\phi(z_{j_{\ell}})|)^2
$$
It is shown in \cite{Barron2018ApproximationNetworks} that

$$
\frac{1}{n}\sum_{x\in S}\E_{\tilde p}(\sum_{j_{\ell}} \tilde p_{j_{\ell}}|\phi(\tilde z_{j_{\ell}})-\phi(z_{j_{\ell}})|)^2 \leq \frac{1}{M}(\sum_{j_{\ell}}\sigma_{j_{\ell}}\sqrt{p_{j_{\ell}}})^2,
$$
where 
$$ \sigma_{j_\ell}^2 = \frac{1}{n}\sum_{x\in S}\sigma^2_{j_{\ell}}(x') = \frac{1}{n}\sum_{x\in S}\sum_{j_{\ell-1}} p_{j_{\ell-1} | j_{\ell}}(x_{j_{\ell-1}}(x')-z_{j_{\ell}})^2 $$ 
and $ z_{j_{\ell}} = \sum_{j_{\ell-1}} p_{j_{\ell-1} | j_{\ell}}x_{j_{\ell-1}} $ are the variance and mean, respectively, of $ x_{\tj_{\ell-1}} $ resulting from a single draw $\tj_{\ell-1} \sim p_{j_{\ell-1}|j_{\ell}} $. 
Bounding the latter term further using Jensen's inequality, we get
\begin{align*}
\sigma_{j_\ell}^2 &\leq \frac{1}{n}\sum_{x\in S}x_{j_{\ell-1}}^2(x')\\ 
&\leq \frac{1}{n}\sum_{x\in S} (\sum_{j_0}p_{j_0|j_{\ell-1}}|x_{j_0}'|)^2\\
&\leq \Big(\frac{1}{n}\sum_{x\in S}\sum_{j_0}p_{j_0|j_{\ell-1}}|x_{j_0}'|^{q^*}\Big)^{2/q^*}
\end{align*}
which is at most 1 by the same reasoning as in the case of $\ell=L$. Hence we obtain
$$
E_\ell^2 \leq \frac{1}{M}(\sum_{j_\ell}\sqrt{p_{j_\ell}})^2 = \frac{1}{M}(e^{\frac{1}{2}H_{1/2}(p_\ell)})^2
$$
Substituting this (as well as the bound of $\frac{1}{M}$ in the $\ell=L$ case) into (\ref{eqn:sum-bound}), we obtain 
\begin{align*}
\E_{\tilde{p}}\Big[\frac{1}{n}\sum_{x\in S}\|f(x;p)-f(x;\tilde{p})\|_2^2 \Big] &\leq \Big(\sum_{\ell} E_\ell\Big)^2 \\
&\leq \Big(\frac{1}{\sqrt{M}}+\sum_{\ell=1}^{L-1}\frac{e^{\frac{1}{2}H_{1/2}(p_\ell)}}{\sqrt{M}}\Big)^2\\
&= \frac{L^2\zeta^2}{M}
\end{align*}
Hence, multiplying both sides by $\mathscr{V}^2$, we get
\begin{align*}
    \E_{\tilde{p}}\Big[\frac{1}{n}\sum_{x\in S}\|f(x;W)-f(x;\widetilde{W})\|_2^2 \Big] \leq \frac{\mathscr{V}^2\zeta^2L^2}{M}.
\end{align*}
\end{proof}

\begin{theorem}
Suppose $ k = 1 $, $ L = 2 $, $ x \in \{-1,+1\}^d $, and $ \P(x_{\tj_0}=+1|j_1) = \sum_{j_0:x_{j_0}=+1}p_{j_0|j_1} = 1/2 $ and $ \P(x_{\tj_0}=-1|j_1) = \sum_{j_0:x_{j_0}=-1}p_{j_0|j_1} = 1/2 $ for all $ j_1 $. Then, for sufficiently large $ M $,
$$
\E_{\tilde p}[|f(x; p) - f(x; \tilde p)|^2] \geq \frac{\zeta^2_1}{32M},
$$
where $ \xi_1 = \frac{1}{2}(1+\sum_{j_1}\sqrt{p_{j_1}}) $.
\label{sampling-lb-appendix}
\end{theorem}
\begin{remark}
Note that the assumptions are satisfied if, for example, $ d $ is even, $p_{j_0|j_1} = 1/d $, and half of the coordinates of $ x $ are $ + 1 $ and the other half are $ - 1 $ (there are $ \binom{d}{d/2} $ ways of choosing $ x $ in this way).
\end{remark}
\begin{proof}
By the bias-variance decomposition,
$$
\E_{\tilde p}[|f(x; p) - f(x; \tilde p)|^2] = |f(x; p) - \E_{\tilde p}[f(x; \tilde p)]|^2 + \text{VAR}_{\tilde p}[f(x; \tilde p)].
$$
The assumptions imply that $ f(x; p) = 0 $. Hence,
$$
\E_{\tilde p}|f(x; p) - f(x; \tilde p)|^2 \geq |\E_{\tilde p}[f(x; \tilde p)]|^2
$$
Using the identity $ \phi(z) = (z+|z|)/2 $ and unbiasedness, we have
$$
\E_{\tilde p}[f(x; \tilde p)] = \sum_{j_1}\E_{\tilde p}[\tilde p_{j_1}\phi\big(\sum_{j_0}\tilde p_{j_0|j_1}x_{j_0}\big)] = \frac{1}{2}\sum_{j_1}\E_{\tilde p}[\tilde p_{j_1}\big|\sum_{j_0}\tilde p_{j_0|j_1}x_{j_0}\big|].
$$
Next, using $ \P(x_{\tj_0}=+1|j_1) = \P(x_{\tj_0}=-1|j_1) = 1/2 $, we have
$$
\E[\big|\sum_{j_0}\tilde p_{j_0|j_1}x_{j_0}\big| | K_{j_1}] = \frac{1}{K_{j_1}}\E[\big|\sum_{i=1}^{K_{j_1}}\epsilon_i\big|],
$$
where $\epsilon_i \stackrel{iid}{\sim} $ Unif$\{-1,1\}$.
By Khintchine's inequality,
$$
\E[\big|\sum_{i=1}^{K_{j_1}}\epsilon_i\big|] \geq \sqrt{K_{j_1}/2}.
$$
Thus, since $ \tilde p_{j_1} = K_{j_1}/M $, we have
$$
\E_{\tilde p}[\tilde p_{j_1}\big|\sum_{j_0}\tilde p_{j_0|j_1}x_{j_0}\big|] \geq \frac{1}{\sqrt{2}M}\E[\sqrt{K_{j_1}}].
$$
Using a Taylor expansion of $ z \mapsto \sqrt{z} $, it can be shown that $ \E[\sqrt{K_{j_1}}] \geq \sqrt{Mp_{j_1}}-\frac{1-p_{j_1}}{2\sqrt{Mp_1}} $. The lower bound on $ \E_{\tilde p}[f(x; \tilde p)] $ is then
$$
\frac{1}{2}\sum_{j_1}\E_{\tilde p}[\tilde p_{j_1}\big|\sum_{j_0}\tilde p_{j_0|j_1}x_{j_0}\big|] \geq \frac{1}{2\sqrt{2M}}\bigg(\sum_{j_1}\sqrt{p_{j_1}}-\frac{1}{2M}\sum_{j_1}\frac{1-p_{j_1}}{\sqrt{p_1}}\bigg).
$$
For $ M $ sufficiently large, this expression is at least $ \frac{\zeta_1}{4\sqrt{2M}} $, thus proving the claim.
\end{proof}

\begin{theorem}
The number of networks $f(x; \tilde p)$ that arise from the sampling scheme is at most $ 8^{ML}(d e)^M $. Thus, the log-cardinality of the representor set is bounded by $ M(\log(d e)+L\log (8)) $.
\label{thm:cardinality-appendix}
\end{theorem}
\begin{proof}
The proof makes use of the following fact. Let $I(k)$ denote the number of integer partitions of integers equal to $ k $. Then $ I(k) \leq 2^k $.
%is at most $ e^{c\sqrt{k}} $ (\cite{deAzevedoPribitkin2009SimpleFunctions}).
%\begin{enumerate}
%    \item The number of integer partitions of integers less than or equal to $ M $ is at most $ (2/c)M^{1/4}e^{c\sqrt{M}} $.
%    \item The number of $k$-tuples of non-negative integers whose sum is $ M $ is equal to $ \binom{k+M}{M} $.
%\end{enumerate}

%To see this, a recent paper [cite] proves that the number of integer partitions of $ k $ is at most $ k^{-3/4}e^{c\sqrt{k}} $. Hence, by an integral comparison argument, the number of integer partitions of integers less than or equal to $ M $ is at most
%$$
%\int_1^M x^{-3/4}e^{c\sqrt{x}}dx.
%$$
%This integral is less than $ M^{1/4}\int_1^M x^{-1/2}e^{c\sqrt{x}}dx = (2/c)M^{1/4}(e^{c\sqrt{M}}-e^{c}) $, which gives the desired bound of $ (2/c)M^{1/4}e^{c\sqrt{M}} $.

%To show (2), we note that number of $d$-tuples of non-negative integers whose sum is $ k $ equals $ \binom{d+k-1}{k} $. Summing over $ k $ gives the claimed count, since $ \binom{d+M}{M} = \sum_{k=0}^M\binom{d+k-1}{k} $.

%It is instructive to first prove the result for a single-hidden layer network with a single output. 
We will prove the claim by induction. Let $ L = 2 $. In this case, $ \tilde x_{j_2} = f(x; \tilde p)_{j_2} $ has the form
\begin{equation} \label{eq:sum1}
\tilde x_{j_2} = \phi\big(\sum_{j_1}\tilde p_{j_2}\tilde p_{j_1|j_2}\tilde x_{j_1}\big),
\end{equation}
where $ \tilde x_{j_1} = \phi\big(\sum_{j_0}\tilde p_{j_0|j_1}x_{j_0}\big) $. Let us now count the number of vectors $ (\tilde x_{j_1}) $.
Note that for each $ j_1 $, the number of outputs $ \tilde x_{j_1} $ is the number of nonnegative integers $ K_{j_0,j_1} $ that sum to $ K_{j_1} $, or $ \binom{K_{j_1}+d_0-1}{K_{j_1}} $. Thus, for a fixed sequence of integers $ (K_{j_1}) $ that sum to $ M $, there are
$$
\prod_{j_1}\binom{K_{j_1}+d_0-1}{K_{j_1}}
$$
vectors $ (\tilde x_{j_1}) $. Summing over all integers $ K_{j_1} $ that sum to $ M $ yields that the number of vectors $ (\tilde x_{j_1}) $ is
$$
N_1 = \sum_{(K_{j_1}):\sum_{j_1}K_{j_1}=M}\prod_{j_1}\binom{K_{j_1}+d_0-1}{K_{j_1}}.
$$

%This is at most the number of non-negative integers $ K_{j_0,j_1} $ that sum to $ M $, or $ \binom{M+d_0d_1-1}{M} $. However, by relabeling the indices $ j_1 $, we can assume that $ d_1 = M $ (since at most $ M $ of the $ \tilde p_{j_1|j_2} $ are nonzero).

%Actually, the count is exactly 

%$$
%\sum_{K_{j_1}:\sum_{j_1}K_{j_1}=M}\prod_{j_1}\binom{K_{j_1}+d_0-1}{K_{j_1}}
%$$

%Hence, this number is at most $ \binom{M+Md_0-1}{M} \leq (2ed)^M $ (recall that $ d_0 = d $). 
%Now, by permutation invariance, 
%$$
%\sum_{j_1}\tilde p_{j_1}\tilde x_{j_1} = \sum_{j_1}\tilde p_{\tau(j_1)}\tilde x_{\tau(j_1)},
%$$
%where $\tau : [d_1] \mapsto [d_1] $ is a fixed permutation. 
%Note that all (ordered) realizations of sequences $ (\tilde x_{\tau(j_1)} : j_1 \in [d_1]) $ are included in the count $ \binom{M+Md_0-1}{M} $ (At least, I think they are!). 
Next, note that each $ \tilde p_{j_2}\tilde p_{j_1|j_2} = \tilde p_{j_1,j_2} $ is built from counts $ K_{j_{1},j_{2}} $ that sum to $ K_{j_2} $ for each fixed $ j_{2} $. By permutation invariance of the sum \eqref{eq:sum1}, for a fixed nonnegative integer $ K_{j_2} $ and vector $ (\tilde x_{j_1}) $, each output $ \tilde x_{j_2} $ provides at most $ I(K_{j_{2}}) $ different networks. Hence, for a fixed vector $ (\tilde x_{j_1}) $, since the $ K_{j_2} $ sum to $ M $, the number of vectors $ (\tilde x_{j_{2}}) $ is
$$
N_2 = \sum_{(K_{j_2}):\sum_{j_2}K_{j_2}=M}\prod_{j_{2}}I(K_{j_{2}}).
$$
Hence the total number of vectors $ (\tilde x_{j_2}) $ is $ N_1N_2 $.

For general $ L $, consider $ \tilde x_{j_L} = f(x; \tilde p)_{j_L}$, i.e.,
\begin{equation} \label{eq:sum2}
\tilde x_{j_L} = \phi\big(\sum_{j_{L-1}}\tilde p_{j_L}\tilde p_{j_{L-1}|j_{L}}\tilde x_{j_{L-1}}\big).
\end{equation}
Note that each $ \tilde p_{j_L}\tilde p_{j_{L-1}|j_{L}} = \tilde p_{j_{L-1},j_L} $ is built from counts $ K_{j_{L-1},j_{L}} $ that sum to $ K_{j_{L}} $ for each fixed $ j_{L} $. By permutation invariance of the sum in \eqref{eq:sum2}, for a fixed nonnegative integer $ K_L $ and vector $ (\tilde x_{j_{L-1}}) $, each output $ \tilde x_{j_L} $ provides at most $ I(K_{j_{L}}) $ different networks. Hence, for a fixed vector $ (\tilde x_{j_{L-1}}) $, since the $ K_{j_L} $ sum to $ M $, the number of vectors $(\tilde x_{j_{L}})$ is
$$
N_L = \sum_{(K_{j_L}):\sum_{j_L}K_{j_L}=M}\prod_{j_L}I(K_{j_{L}}).
$$

%$$
%\prod_{j_{L}}I(K_{j_{L}}) \leq e^{c\sum_{j_{L}}\sqrt{K_{j_{L}}}} \leq %e^{cM},
%$$

%times the number of non-negative integer solutions $ K_{j_{L}} $ to $ \sum_{j_{L}}K_{j_{L}} = M $, or $ \binom{M+M-1}{M} \leq (2e)^M $ (recall that we can replace $ d_{L} $ by $ M $) times the number of realizations of sequences $(\tilde x_{j_{L-1}})$. 

By the induction step, the number of vectors $(\tilde x_{j_{L-1}})$ (each vector $ (\tilde x_{j_{L-1}}) $ is a depth $ L-1 $ network with $d_{L-1}$ output nodes) is $ N_1N_2\cdots N_{L-1} $. Hence, the total count is
$ N_1N_2\cdots N_{L} $.

Since at most $ M $ of the $ \tilde p_{j_{\ell}|j_{\ell+1}} $ are nonzero, by relabeling the indices $ j_{\ell} $ in each layer $ \ell = 1, 2, \dots, L $, we can assume that $ d_{\ell} = M $. This means that 
$$ N := N_2 = N_3 = \cdots = N_L = \sum_{(K_j):\sum_{j=1}^MK_j=M}\prod_{j}I(K_j).
$$
Next, note that $ I(K_j) \leq 2^{K_j} $ and hence $ \prod_{j}I(K_j) \leq 2^M $. Furthermore, $ \sum_{(K_j):\sum_{j=1}^MK_j=M}1 = \binom{2M-1}{M} \leq 4^M $. Thus, $ N \leq 8^M $. As for $ N_1 $, we note that $ \binom{K_{j_1}+d_0-1}{K_{j_1}} \leq (2ed_0)^{K_{j_1}} $ and hence $ N_1 \leq 4^M(2ed_0)^M = (8ed_0)^M $. This shows that
$$
N_1N_2\cdots N_{L} \leq 8^{M(L-1)}(8ed_0)^M = 8^{ML}(d_0e)^M.
$$

%at most $ e^{cM}(2e)^M\times (e^{cM}(2e)^M)^{L-1}(2ed)^M = (e^{cM}(2e)^M)^{L}(2ed)^M $, which proves the claim.
\end{proof}

The following Corollary, mentioned in the main text, allows us to remove dependence on $d$ when we have $\ell_2$ constraints on the data. 
\begin{corollary}
Let $ \mathcal{S} = \text{span}(S) $ denote the subspace spanned by $ S $.
%(i.e., the rowspace of $ X $, where $X$ is the dataset $S$ arranged as an $n\times d$ matrix). 
%Suppose $ d \geq n $ and let $ W'_1 = \text{proj}_{\mathcal{X}}(W_1) = W_1X^T(XX^T)^{-1}X $ denote the orthogonal projection of the rowspace of $ W_1 $ onto $ \mathcal{S} $.
Let $ W'_1 = \text{proj}_{\mathcal{S}}(W_1) $ denote the orthogonal projection of the rowspace of $ W_1 $ onto $ \mathcal{S} $. Let $ p' $ denote the path distribution induced by weight matrices $(W'_1,W_2, \dots,W_L) $. The number of networks $f(x; \tilde p')$ evaluated at the training data $ S $ is at most $ 8^{ML}(n e)^M $. Thus, the log-cardinality of the representor set is bounded by $ M(\log(n e)+L\log (8)) $.
\label{thm:effictive-dim-appendix}
\end{corollary}
\begin{proof}
The effective input dimension of $ f(x; p') $, acted on $n$ data points $S$, is at most $n$. Hence, we obtain the conclusion from the previous lemma.
\end{proof}
To get the metric entropy bound that removes dependence on $ d $, we first note that $ f(x; p') = f(x; p) $ for $ x\in S $. 
%Hence the effective input dimension of $ f(x; p') $, acted on $n$ data %points $S$, is at most $n$. 
Furthermore, because an orthogonal projection is a bounded operator, if $ \|x\|_2 \leq r $, then  $\mathscr{V}_2$ defined in terms of $(W'_1,W_2, \dots,W_L) $ can be bounded by the same quantities in terms of $(W_1,W_2, \dots,W_L) $, i.e., $ \|W_L\cdots W_2 W'_1w_0\| \leq r\|W_L\cdots W_2 W_1\| $. These facts imply that an empirical cover of $ \mathcal{V}'f(x; p') $ is also an empirical cover of $ \mathcal{V}f(x; p) $ for $ x \in S $.

\begin{corollary}
Let $\epsilon,\gamma >0$, $1\leq q \leq 2$. Then 
\begin{align}
    \log\mathcal{N}_2(\epsilon,\mathcal{F}_\gamma(\mathscr{V}_q,\zeta_q),S) \leq \frac{9\mathscr{V}_q^2\zeta_q^2L^2(L+\log(de))}{\gamma^2\epsilon^2}
    \end{align}
\label{thm:covering-number}
\end{corollary}
\begin{proof}
We first observe that $R_\gamma$ is $\frac{1}{\gamma}$ Lipschitz. Moreover, Lemma A.2 in \cite{Bartlett2017} shows that for any $j$, $\mathcal{M}(\cdot,j)$ is 2-Lipschitz with respect to $\|\cdot\|_\infty$. Then for any network $f(x;W) \in \mathcal{F}(\mathscr{V}_q,\zeta_q)$, by Theorem \ref{thm:maximal-bound-appendix}, we have that there exists $f(x;\widetilde{W})$ such that $n^{-1}\sum_{x\in S}\|f(x;W)-f(x;\widetilde{W})\|_2^2 \leq \Big(\frac{\mathscr{V}_{q}\zeta_q L}{\sqrt{M}}\Big)^2$. Then
\begin{align*}
&\frac{1}{n}\sum_{(x,y):x\in S}|R_\gamma(-\mathcal{M}(f(x;W),y)) - R_\gamma(-\mathcal{M}(f(x;\widetilde{W}),y))|^2 \\&\leq \frac{1}{n}\sum_{(x,y):x\in S}\frac{1}{\gamma^2}|\mathcal{M}(f(x;W),y)  -\mathcal{M}(f(x;\widetilde{W}),y)|^2 \\
&\leq \frac{1}{n}\sum_{(x,y):x\in S}\frac{4}{\gamma^2}\|f(x;W) - f(x;\widetilde{W})\|^2_2\\
&\leq \Big(\frac{2\mathscr{V}_{q}\zeta_q L}{\gamma \sqrt{M}}\Big)^2
\end{align*}
The results is thus an immediate consequence of Theorem \ref{thm:cardinality-appendix} with $M_\epsilon = \big(\frac{2\mathcal{V}_q\zeta_qL}{\gamma \epsilon}\big)^2$. Note that the factor of $9$ arises from the additional factor of $\log(8)$ in the cardinality bound.
\end{proof}

\begin{theorem}
Let $f(x;W)$ be an $L$-layer positive homogeneous network and let $\delta \in (0,1)$. 
For any $1\leq q \leq 2$ and $\gamma >0$, with probability at least $1-\delta$ over the training set $S$, the generalization error $\ell(f) - \hat{\ell}_\gamma(f)$ is bounded by
    \begin{align}
        \widetilde{\mathcal{O}}\Big(\frac{ \mathscr{V}_q\zeta_qL\sqrt{L+\log(d)}}{\gamma \sqrt{n}}+\sqrt{\frac{\log(1/\delta)}{n}} \Big),
    \end{align}
where $\mathscr{V}_q,\zeta_q$ are the $q$- path variation and path complexity of $f$. 
\label{thm:gen-bound-appendix}
\end{theorem}

To prove this, we first prove the generalization bound for the classes $\mathcal{F}_\gamma(\mathscr{V}_q,\zeta_q)$ with \emph{a priori} bounded path variation and path complexity, and then take a union bound to obtain a \emph{post hoc} guarantee.

We first recall the \emph{empirical Rademacher complexity} of a class of real-valued functions $\mathcal{G}$ with respect to a dataset $S = \{x^1,...,x^n\}$:
$$
\widehat{\mathscr{R}}_S(\mathcal{G}) = \E_{\epsilon}\Big[\sup_{g\in \mathscr{G}} \frac{1}{n}\sum_{i=1}^n \epsilon_i g(x^i)\Big]
$$
where $\epsilon_i \stackrel{iid}{\sim} $ Unif$\{-1,1\}$. For our purposes, the utility of the empirical Rademacher complexity is captured by the following standard result.
\begin{lemma}[\cite{Mohri2018}]
Let $\mathcal{G}$ be a class of functions with values in $[0,1]$. Then for any $\delta >0$, with probability at least $1-\delta$ over $S$, for all $g\in \mathcal{G}$ we have
\begin{align*}
    \E[g(x)] \leq \frac{1}{n}\sum_{i=1}^ng(x^i) + 2\widehat{\mathscr{R}}_S(\mathcal{G}) + 3\sqrt{\frac{\log(2/\delta)}{2n}}
\end{align*}
\label{thm:rademacher-bound}
\end{lemma}
To bound to empirical Rademacher complexity, we use a standard bound via a Dudley entropy integral.
\begin{lemma}[Dudley entropy integral; see e.g. note by \cite{SridharanNoteBound}]
For a class of functions $\mathcal{G}$ with values in $[0,1]$ and a dataset $S$ of $n$ points, we have
\begin{align*}
    \widehat{\mathscr{R}}_S(\mathcal{G}) \leq \inf_{\alpha \geq 0}\left[4\alpha + 12\int_{\alpha}^1 \sqrt{\frac{\log\mathcal{N}(\epsilon, \mathcal{G}, S)}{n}}d\epsilon \right]
\end{align*}
\label{thm:dudley-integral}
\end{lemma}

Using these results with Lemma \ref{thm:covering-number}, we may obtain the following. 

\begin{lemma}
Let $\delta \in (0,1)$, $\gamma >0$, $1\leq q\leq 2$. Then with probability at least $1-\delta$ over an i.i.d. draw of $S$ , we have for all $f\in \mathcal{F}_\gamma(\mathscr{V}_q, \zeta_q)$
\begin{align}
    \ell(f) \leq \hat{\ell}_\gamma(f) + \frac{8}{n}+ \frac{48\mathscr{V}_q\zeta_q L\sqrt{L+\log(ed)}\log(n)}{\gamma \sqrt{n}} + 3\sqrt{\frac{\log(2/\delta)}{2n}}
    \label{eqn:apriori-bound}
\end{align}
\label{thm:apriori-bound}
\end{lemma}
\begin{proof}
Define 
$$
A^2 = \frac{ 4 \mathscr{V}_q^2\zeta_q^2L^2(L+\log(ed))}{\gamma^2 n}
$$
so that $\log \mathcal{N}_2(\epsilon, \mathcal{F}_\gamma(\mathscr{V}_q,\zeta_q), S)/n = \frac{A^2}{\epsilon^2}$. Then by Lemma \ref{thm:dudley-integral}, we have that
\begin{align*}
    \widehat{\mathscr{R}}_S(\mathcal{F}_\gamma(\mathscr{V}_q,\zeta_q)) \leq \inf_{\alpha \geq 0}\left[4\alpha + 12A\int_{\alpha}^1 \frac{1}{\epsilon}d\epsilon \right] = \inf_{\alpha \geq 0}\left[4\alpha + 12A\log(1/\alpha) \right]
\end{align*}
It is easy to verify that the above expression is minimized at $\alpha = 3A$, though to keep the expression somewhat cleaner, we use to choice $\alpha = \frac{1}{n}$, which produces the bound 
\begin{align*}
    \widehat{\mathscr{R}}_S(\mathcal{F}_\gamma(\mathscr{V}_q,\zeta_q)) \leq \frac{4}{n}+A\log(n) = \frac{4}{n} + \frac{24\mathscr{V}_q\zeta_q L\sqrt{L+\log(ed)}\log(n)}{\gamma \sqrt{n}}.
\end{align*}
The result follows from Lemma \ref{thm:rademacher-bound} together with the fact that $\ell(f) \leq \E[R_\gamma(f(x;W),y)]$ and $\frac{1}{n}\sum_{(x,y):x\in S}R_\gamma(f(x;W),y) \leq \hat{\ell}_\gamma(f)$.
\end{proof}

The above gives a generalization guarantee for the class $\mathcal{F}_{\gamma}(\mathscr{V}_q,\zeta_q)$ with \emph{a priori} bounded path variation and path complexity, and given $\gamma >0$. Namely, it gives a statement of the form
\begin{center}
    \emph{$\forall$ classes $\mathcal{F}_{\gamma}(\mathscr{V}_q,\zeta_q)$ we have with probability at least $1-\delta$ over the training set $S$ that $\forall$ $f\in \mathcal{F}_{\gamma}(\mathscr{V}_q,\zeta_q)$ the bound (\ref{eqn:apriori-bound}) holds.}
\end{center}

However, in practice, we do not have such guarantees on the on the size of these quantities before seeing the data. In order to obtain post hoc guarantees for a network $f(x;W)$, we need to prove a statement of the form
\begin{center}
    \emph{With probability at least $1-\delta$ over the training set $S$, $\forall$ classes $\mathcal{F}_{\gamma}(\mathscr{V}_q,\zeta_q)$ we have that $\forall$ $f\in \mathcal{F}_{\gamma}(\mathscr{V}_q,\zeta_q)$ the bound (\ref{eqn:apriori-bound}) holds.}
\end{center}
To prove a statement of the latter form, we must instead instantiate the above bound for many values of $\mathscr{V},\zeta,\gamma$ and take a union bound. The below approach to doing so is similar to that of \cite{Bartlett2017}.

\begin{proof}[Proof of Theorem \ref{thm:gen-bound-appendix}]
Given integers $(j_1,j_2,j_3)$, define the instances
\begin{align*}
    \mathcal{B}(j_1,j_2,j_3) = \Big\{(\gamma, S, W) : 0 < \frac{1}{\gamma} < \frac{2^{j_1}}{\sqrt{n}},\,\, \mathscr{V}_q(W) \leq  j_2,\,\, \zeta_q(W)\leq j_3\Big\}
\end{align*}
And for $\delta \in (0,1)$, divide $\delta$ as
\begin{align*}
    \delta(j_1,j_2,j_3) = \frac{\delta}{2^{j_1}j_2(j_2+1)j_3(j_3+1)}
\end{align*}
so that by construction $\sum_{j_1,j_2,j_3\in \mathbb{N}} \delta(j_1,j_2,j_3) = \delta$. Then by Lemma \ref{thm:apriori-bound}, we have that for every $(j_1,j_2,j_3) \in \mathbb{N}^3$, we have with probability at least $1-\delta(j_1,j_2,j_3)$ that for all instances $(\gamma, S, W) \in \mathcal{B}(j_1,j_2,j_3)$,
\begin{align}
    \ell(f) &\leq \hat{\ell}_\gamma(f) + \frac{8}{n}+ \frac{48\cdot 2^{j_1}\cdot j_2\cdot j_3\cdot  L\sqrt{L+\log(ed)}\log(n)}{n} + 3\sqrt{\frac{\log(2/\delta(j_1,j_2,j_3))}{2n}}\label{eqn:discrete-bound}\\
    &\leq \hat{\ell}_\gamma(f) + \frac{8}{n}+ \frac{48\cdot 2^{j_1}\cdot j_2\cdot j_3\cdot L\sqrt{L+\log(ed)}\log(n)}{n} \\ & + 3\sqrt{\frac{\log(2/\delta)+ \log(2^{j_1}) + 2\log(j_2+1) + 2\log(j_3+1) }{2n}}
    \label{eqn:discrete-bound-end}
\end{align}
Then taking a union bound over the integers $(j_1,j_2,j_3)$, we have that (\ref{eqn:discrete-bound}-\ref{eqn:discrete-bound-end}) holds simultaneously over all $\mathcal{B}(j_1,j_2,j_3)$ with probability at least $1-\delta$. Then for a \emph{given} $\gamma, X$, and $f(x;W)$ with path variation and complexity $\mathscr{V}_q,\zeta_q$, Let $j_1^*,j_2^*,j_3^*$ be the smallest integers such that $\frac{1}{\gamma}\leq \frac{2^{j_1^*}}{\sqrt{n}}, \mathscr{V}_q \leq j_2^*$, and $\zeta_q \leq j_3^*$. Then we have by definition that $2^{j_1^*} \leq \frac{2\sqrt{n}}{\gamma}$, $j_2^*\leq \mathscr{V}_q + 1$ and $j_3^* \leq \zeta_q+1$. Plugging these values in and cleaning up with the notation $\widetilde{\mathcal{O}}$ yields the stated result.
\end{proof}

\section{Additional results mentioned in the main text}
\subsection{Bounding normalizing constants}
\label{bounding-normalizing-constants}
\begin{lemma}[Induced norms as normalizing constants]
Let $1\leq q\leq \infty$  and define $w_{j_0} = (n^{-1}\sum_{x \in S}|x_{j_0}|^{q})^{1/{q}}$. Then
\begin{align*}
    \sum_{j_0,j_1,...,j_L}w_{j_0}w_{j_0,j_1,...,j_L} \leq (\max_{x\in S}\|x\|_q)k^{1-1/q}\|W_LW_{L-1}\cdots W_1\|_q
\end{align*}
where $\|\cdot\|_q$ is the matrix norm induced by the vector $q$ norm.
\end{lemma}
\label{thm:induced-norm-appendix}
\begin{proof}
We observe that this is the same as showing that $$\|W_LW_{L-1}\cdots W_1 w_0\|_1 \leq rk^{1-1/q}\|W_LW_{L-1}\cdots W_1w_0\|_q$$ where $w_0 = [w_1,w_2,...,w_d]^\top$ (the vector containing the values $w_{j_0}$). Notice that $$W_LW_{L-1}\cdots W_1 w_0$$ is simply a vector in $\R^k$, and hence by an application of H\"older's inequality, we have
$$
\|W_LW_{L-1}\cdots W_1 w_0\|_1 \leq k^{1-1/q} \|W_LW_{L-1}\cdots W_1 w_0\|_q
$$
Since the vector $q$ norm induces the mstrix $q$ norm, we have that this is at most $$
k^{1-1/q}\|W_LW_{L-1}\cdots W_1\|_q \|w_0\|_q \leq rk^{1-1/q}\|W_LW_{L-1}\cdots W_1\|_q.
$$
\end{proof}

\begin{lemma}[$(q,1)$ norms as normalizing constants]
Let $1\leq q \leq \infty$, and let $q^*$ be its conjugate exponent. Then 
\begin{align*}
   \mathscr{V}_q \leq (\max_{x\in S} \|x\|_{q^*})\Big\|\prod_1^L |W_\ell|\Big\|_{q,1} .
\end{align*}
\end{lemma}
\begin{proof}
We observe from H\"older's inequality,
\begin{align*}
 \mathscr{V}_q &= \sum_{j_L}\sum_{j_0}w_{j_0}\sum_{j_1,...,j_{L-1}}w_{j_0,...,j_L} \\
     &\leq \sum_{j_L}\Big(\sum_{j_0}\Big(\sum_{j_1,...,j_{L-1}}w_{j_0,...,j_L}\Big)^q\Big)^{1/q}\Big(\sum_{j_0}w_{j_0}^{q^*}\Big)^{1/q^*}\\
     &\leq (\max_{x\in S}\|x\|_{q^*}) \Big\|\prod_1^L |W_\ell|\Big\|_{q,1}.
\end{align*}
\end{proof}

\begin{lemma}
We have
$$
\mathscr{V}_2 \leq \sum_{j_L}\big(\sum_{j_0,j_1,...,j_{L-1}} w_{j_0,j_1,...,j_L}^2\big)^{1/2}
$$
where in the single output case, the right-hand side is equal to the $2-$path norm $\phi_2$ from \cite{Neyshabur2015Norm-BasedNetworks}.
\end{lemma}
\begin{proof}
This can be seen by repeated application of the Cauchy-Schwarz inequality, as follows. Assume, without loss of generality, that $r=1$, so that $S \subseteq \mathbb{B}_2(1)$. Then we have
\begin{align*}
 \sum_{(j_0,j_1,\dots,j_L)}w_{j_0}w_{j_0,j_1,\dots, j_L}  = \sum_{(j_1,j_2,\dots,j_L)}w_{j_1,j_2,\dots,j_L}\sum_{j_0}w_{j_0}w_{j_0,j_1}.
\end{align*}
Then, for each $ j_1 $, we apply the Cauchy-Schwarz inequality to the sum $ \sum_{j_0}w_{j_0}w_{j_0,j_1} $, yielding the bound
$$
\sum_{(j_1,j_2,\dots,j_L)}w_{j_1,j_2,\dots,j_L}(\sum_{j_0}w^2_{j_1,j_0})^{1/2}(\sum_{j_0}w^2_{j_0})^{1/2}.
$$
By the $ \ell_2 $ condition on the inputs, $ (\sum_{j_0}w^2_{j_0})^{1/2} \leq r $. 
Continuing similarly, we have for each $ \ell = 1, 2, \dots, L $,
\begin{align*}
& \sum_{(j_{\ell},j_{\ell+1},\dots,j_L)}w_{j_{\ell},j_{\ell+1},\dots, j_L}(\sum_{(j_0,j_1,\dots,j_{\ell-1})}w^2_{j_0,j_1,\dots,j_{\ell}})^{1/2} \\ &= \sum_{(j_{\ell+1},j_{\ell+2},\dots,j_L)}w_{j_{\ell+1},j_{\ell+2},\dots,j_L} \times  \sum_{j_{\ell}}w_{j_{\ell+1},j_{\ell}}(\sum_{(j_0,j_1,\dots,j_{\ell-1})}w^2_{j_0,j_1,\dots,j_{\ell}})^{1/2}.
\end{align*}
As before, for each $ j_{\ell+1} $, we apply the Cauchy-Schwarz inequality to the sum $$ \sum_{j_{\ell}}w_{j_{\ell+1},j_{\ell}}(\sum_{(j_0,j_1,\dots,j_{\ell-1})}w^2_{j_0,j_1,\dots,j_{\ell}})^{1/2}, $$ yielding the bound
$$
\sum_{(j_{\ell+1},j_{\ell+2},\dots,j_L)}w_{j_{\ell+1},j_{\ell+2},\dots, j_L}(\sum_{(j_0,j_1,\dots,j_{\ell})}w^2_{j_0,j_1,\dots,j_{\ell+1}})^{1/2}.
$$
Repeating this procedure to $\ell = L-1$, we may obtain the stated bound.
\end{proof}

\subsection{Details of pooling case}
\label{pooling-details}
We begin with some basic properties of the max/average pooling operator $\mathcal{P}_\mathcal{Z}: \mathcal{X} \rightarrow \mathcal{Y}$, where $\mathcal{X},\mathcal{Y}$ are vector spaces with dimension $d_\mathcal{X}, d_\mathcal{Y}$ and $\mathcal{Z}$ is a collection of subsets $\{Z_1,...,Z_{d_\mathcal{Y}}\}$ where $Z_i \subset \{1,...,d_\mathcal{X}\}$. For a given input $X \in \mathcal{X}$, $\mathcal{P}_{\mathcal{Z}}$ computes $[\mathcal{P}_{\mathcal{Z}}(X)_1,...,\mathcal{P}_{\mathcal{Z}}(X)_{d_{\mathcal{Y}}}]^T$ where 
\begin{align*}
\mathcal{P}^{\mathsf{max}}_{\mathcal{Z}}(X)_i = \max_{j\in Z_i} X_j.
\end{align*}
for max pooling and
\begin{align*}
\mathcal{P}^{\mathsf{avg}}_{\mathcal{Z}}(X)_i = \frac{1}{|Z_i|}\sum_{j\in Z_i} X_j.
\end{align*}
for average pooling. The argument is the same for both, so we simply use $\mathcal{P}$ to denote either max or average pooling.
Now given weight matrices $W_{\ell+1}, W_\ell$ with $W_{\ell+1} \in \R^{d_{\ell+1},d_{\ell}}$ and $W_{\ell}\in \R^{d_{\ell}', d_{\ell-1}}$ with positive entries, a pooling layer can be written as
\begin{align*}
    z_{j_{\ell+1}}(x) = \sum_{j_\ell} w_{j_{\ell+1}, j_\ell} \mathcal{P}_{j_\ell}\big(\phi\big(\sum_{j_{\ell-1}}w_{j_\ell', j_{\ell-1}}x_{j_\ell-1}(x)\big)\big) 
    % = \sum_{j_\ell}  \mathcal{P}_{j_\ell}\big(\phi\big(\sum_{j_{\ell-1}}w_{j_{\ell+1}, j_\ell'}w_{j_\ell', w_{\ell-1}}z_{j_\ell-1}\big)\big)
\end{align*}
To handle the signs of $w_{j_{\ell+1}, j_{\ell}}$, we may now double the number of units $d_\ell$ and have $\mathcal{P}_{\ell}(X)$ compute $[\mathcal{P}_\ell(X)_1,...,\mathcal{P}_\ell(X)_{d_\ell}, - \mathcal{P}_\ell(X)_1,...,-\mathcal{P}_\ell(X)_{d_\ell}]$ and adjust weights accordingly. A typical term in the analysis of Theorem \ref{thm:maximal-bound-appendix} with pooling will now look like %We may need to adjust the sampling scheme a bit because now the unit $\phi\big(\sum_{j_{\ell-1}}w_{j_{\ell+1}, j_\ell'}w_{j_\ell', w_{\ell-1}}z_{j_\ell-1}\big)$ computes a vector in such a way that 
%\begin{align*}
%    |\mathcal{P}_{j_\ell}(X)-\mathcal{P}_{j_\ell}(X')| \leq \max_{j\in Z_{j_\ell}} |X_j-X'_j|
%\end{align*}
%Though perhaps we could just about this by the sum, which will not affect the bound. 
\begin{align*}
    &\sum_{j_L}|f^{\ell+1}(x;p,\tilde{p})_{j_L} - f^\ell(x;p,\tilde{p})_{j_L}|\\ &\leq \sum_{j_L}\sum_{j_{L-1},\dots,j_{\ell+1}} \tilde{p}_{j_L,j_{L-1},\dots,j_{\ell+1}}\Big|\sum_{j_\ell}\tilde{p}_{j_\ell|j_{\ell+1}}(\mathcal{P}_{j_\ell}(\phi(\tilde{z}_{j_\ell}(x))) - \mathcal{P}_{j_\ell}(\phi(z_{j_\ell}(x))))\Big|\\
    &\leq \sum_{j_L}\sum_{j_\ell}\tilde{p}_{j_L,j_\ell}|\mathcal{P}_{j_\ell}(\phi(\tilde{z}_{j_\ell}(x))) - \mathcal{P}_{j_\ell}(\phi(z_{j_\ell}(x)))|\\
    &= \sum_{j_\ell}\tilde{p}_{j_\ell}\Big|\mathcal{P}_{j_\ell}(\phi(\tilde{z}_{j_\ell}(x))) - \mathcal{P}_{j_\ell}(\phi(z_{j_\ell}(x)))\Big|\\
    &\leq \sum_{j_\ell'}\tilde{p}_{j_\ell'}|A_{j_\ell'}(x)|
\end{align*}
where now $A_{j_\ell'}(x) = \sum_{j_{\ell-1}} (\tilde{p}_{j_{\ell-1}|j_\ell'} - p_{j_{\ell-1}|j_\ell'})z_{j_{\ell-1}}(x)$, which is the same as the term appearing in the case without pooling. [Note we just need to define $K_{j_\ell} = \sum_{j_\ell'\in Z_{j_\ell}} K_{j_\ell'}$ in our counts.] 

\subsection{Computational aspects of sampling}
As we mention briefly in the main text, to generate samples from Multinomial$(M,p)$ directly, we would need to store and sample from the full path distribution $p_{j_0,j_1,...,j_L}$, which quickly becomes unwieldy as $L$ grows, since it involves storing a (potentially dense) L-tensor. It turns out, however, that we can store only the conditional distributions, which are just matrices, and can be computed easily from the collection of successive matrix products $ \{W_{\ell}W_{\ell-1}\cdots W_1:\ell=1,2,\dots, L\} $ (the collection itself can be inductively constructed), since
$$
p_{j_\ell|j_{\ell+1}} = w_{j_{\ell+1},j_{\ell}}\frac{\|W_{\ell}[j_{\ell},]W_{\ell-1}\cdots W_1\|_1}{\|W_{\ell+1}[j_{\ell+1},]W_{\ell}\cdots W_1\|_1}
$$
and
$$
p_{j_L} = \frac{\|W_{L}[j_{L},]W_{L-1}\cdots W_1\|_1}{\mathscr{V}},
$$
where $ W_{\ell}[j_{\ell},] $ (resp. $ W_{\ell}[,j_{\ell-1}] $) is row (resp. column) $ j_{\ell} $ (resp. $ j_{\ell-1}$) of $ W_{\ell} $. Thus, the conditional probabilities are reweighted versions of the weight matrices. Given these matrices, a sample $K\sim $ Multinomial$(M,p)$ can be generated in $\mathcal{O}(LM)$ time by repeating the following $M$ times:
\begin{itemize}
    \item Sample $\tj_L \sim p_{j_L}$
    \item Sample $\tj_{L-1} \sim p_{j_{L-1}|\tj_L}$\\\vdots
    \item Sample $\tj_0 \sim p_{j_0|\tj_1}$ 
\end{itemize}

\end{document}